\documentclass[sigconf]{acmart}
\AtBeginDocument{%
  }
\copyrightyear{2025} \acmYear{2025} \setcopyright{cc} \setcctype{by} \acmConference[CCS '25]{Proceedings of the 2025 ACM SIGSAC Conference on Computer and Communications Security}{October 13--17, 2025}{Taipei, Taiwan} \acmBooktitle{Proceedings of the 2025 ACM SIGSAC Conference on Computer and Communications Security (CCS '25), October 13--17, 2025, Taipei, Taiwan}\acmDOI{10.1145/3719027.3765052} \acmISBN{979-8-4007-1525-9/2025/10}
\usepackage{xspace}
\usepackage{bm}
\usepackage{booktabs}
\usepackage{multirow}
\usepackage{subcaption}
\usepackage{enumitem}
\usepackage[linesnumbered, ruled, vlined]{algorithm2e}

\usepackage{bbm}
\newtheorem{theorem}{Theorem}

\newtheorem{lemma}{Lemma}

\newtheorem{remark}{Remark}

\newtheorem{assumption}{Assumption}

{\begin{list}               %
    {$\bullet$ \hfill}{
        \setlength{\leftmargin}{\parindent}
        \setlength{\parsep}{0.04\baselineskip}
        \setlength{\itemsep}{0.5\parsep}
        \setlength{\labelwidth}{\leftmargin}
        \setlength{\labelsep}{0em}}
    }
{\end{list}}

\providecommand{\cref}[1]{Chapter~\ref{#1}}

\providecommand{\R}{\ensuremath{\mathbb{R}}}
\providecommand{\I}{\ensuremath{\mathbb{I}}}

\providecommand{\E}{\ensuremath{\mathbb{E}}}

\renewcommand{\vec}[1]{\ensuremath{\boldsymbol{#1}}}

\providecommand{\calN}{\mathcal{N}}

\providecommand{\mI}{\mathbf{I}}

\providecommand{\vp}{\mathbf{p}}

\providecommand{\vx}{\mathbf{x}}

\providecommand{\vz}{\mathbf{z}}

\providecommand{\vepsilon}{\vec{\epsilon}}

\makeatletter
\DeclareRobustCommand\onedot{\futurelet\@let@token\@onedot}
\def\@onedot{\ifx\@let@token.\else.\null\fi\xspace}

\def\eg{\emph{e.g}\onedot} 
\def\ie{\emph{i.e}\onedot}

\makeatother

\newcommand{\proposed}{PAIA\xspace}

\begin{document}

\title[Concept Auditing for Shared Diffusion Models at Scale]{What Lurks Within? Concept Auditing for Shared Diffusion Models at Scale}

\author{Xiaoyong (Brian) Yuan}
\affiliation{
\institution{Clemson University}
\city{Clemson}
\state{SC}
\country{USA}}
\email{xiaoyon@clemson.edu}

\author{Xiaolong Ma}
\affiliation{
\institution{The University of Arizona}
\city{Tucson}
\state{AZ}
\country{USA}
}
\email{xiaolongma@arizona.edu}

\author{Linke Guo}
\affiliation{%
\institution{Clemson University}
\city{Clemson}
\state{SC}
\country{USA}
}
\email{linkeg@clemson.edu}

\author{Lan (Emily) Zhang}
\affiliation{%
\institution{Clemson University}
\city{Clemson}
\state{SC}
\country{USA}
}
\email{lan7@clemson.edu}

\begin{abstract}
Diffusion models (DMs) have revolutionized text-to-image generation, enabling the creation of highly realistic and customized images from text prompts. With the rise of parameter-efficient fine-tuning (PEFT) techniques like LoRA, users can now customize powerful pre-trained models using minimal computational resources. However, the widespread sharing of fine-tuned DMs on open platforms raises growing ethical and legal concerns, as these models may inadvertently or deliberately generate sensitive or unauthorized content, such as copyrighted material, private individuals, or harmful content. 
Despite increasing regulatory attention on generative AI, there are currently no practical tools for systematically auditing these models before deployment.

In this paper, we address the problem of concept auditing: determining whether a fine-tuned DM has learned to generate a specific target concept. Existing approaches typically rely on prompt-based input crafting and output-based image classification but they suffer from critical limitations, including prompt uncertainty, concept drift, and poor scalability. To overcome these challenges, we introduce Prompt-Agnostic Image-Free Auditing (\proposed), a novel, model-centric concept auditing framework. By treating the DM as the object of inspection, \proposed enables direct analysis of internal model behavior, bypassing the need for optimized prompts or generated images. It integrates two key components: a prompt-agnostic strategy that mitigates prompt sensitivity by analyzing model behavior during late-stage denoising, and an image-free detection method based on conditional calibrated error, which compares the internal dynamics of a fine-tuned model against its base version. 
Our auditing setting assumes internal access to DMs, but does not require access to proprietary fine-tuning data or user prompts, an assumption aligned with how hosted platforms audit uploaded models.
We evaluate \proposed on $320$ controlled models trained with curated concept datasets and $771$ real-world community models sourced from a public DM sharing platform, covering a wide range of concepts including celebrities, cartoon characters, videogame entities, and movie references. Evaluation results show that \proposed achieves over $90$\% detection accuracy while reducing auditing time by $18$ - $40\times$ compared to existing baselines, and remains robust under adaptive attacks. To our knowledge, \proposed is the first scalable and practical solution for pre-deployment concept auditing of diffusion models, providing a practical foundation for safer and more transparent diffusion model sharing\footnote{This is an extended version of the paper accepted at CCS 2025.}.

\end{abstract}



\maketitle

\section{Introduction}
Diffusion models (DMs) have revolutionized text-to-image (T2I) generation, enabling the synthesis of highly realistic and semantically rich images from natural language prompts~\cite{croitoru2023diffusion, ruiz2023dreambooth, zhang2023adding}. Through iteratively refining noise into coherent visual content, DMs have surpassed traditional generative methods such as GANs~\cite{goodfellow2014generative} and VAEs~\cite{kingma2013auto} in both visual fidelity and flexibility. This leap in generative quality has been further accelerated by parameter-efficient fine-tuning (PEFT) methods like Low-Rank Adaptation (LoRA)~\cite{hu2021lora}, which allow users to customize large pre-trained models, such as Stable Diffusion \cite{rombach2021highresolution}, for specific concepts or styles using limited compute and memory resources. 

This shift has catalyzed a thriving ecosystem of community-driven model customization. Users can now fine-tune and distribute their models with minimal technical expertise, aided by user-friendly toolkits~\cite{automatic1111,comfy_website} and supported by sharing platforms such as Civitai~\cite{civitai_website}, HuggingFace~\cite{huggingface_website}, and SeaArt~\cite{seaart_website}. These platforms host tens of thousands of customized DMs, covering a broad spectrum of visual concepts, artistic styles, and application domains. While this democratization has opened new frontiers in creativity and accessibility, it has also introduced significant risks. Fine-tuned DMs may be misused to generate inappropriate or legally problematic content. Studies have documented instances where models replicate copyrighted characters~\cite{zhang2023copyright,shan2023glaze, zhu2024watermark}, impersonate real individuals via deepfakes~\cite{hao2024doesn}, or produce inappropriate content~\cite{zhang2024generate,Yang2023SneakyPromptJT,Chin2023Prompting4DebuggingRT}.

Despite growing concerns, technical oversight remains minimal. Current auditing practices on public model hubs rely heavily on user-supplied tags with little automation or verification. For example, Civitai~\cite{civitai_website}, one of the most widely used platforms, relies primarily on user-provided metadata to flag models, such as depictions of ``real people,'' or ``mature content,'' which are loosely defined, inconsistently applied, and easily circumvented. The absence of standardized review procedures or systematic validation has already led to legal consequences. 
The community reports and discussions suggest that Civitai has received takedown requests related to the unauthorized use of copyrighted or personal content due to the models trained on proprietary content\footnote{Civitai Facilitates the Use of Stolen Intellectual Property, \url{https://luddite.pro/civitai-facilitates-use-stolen-intellectual-property}}\footnote{Square Enix may have filed a DMCA takedown notice with Civitai, \url{https://www.reddit.com/r/StableDiffusion/comments/11zubsj/it_appears_that_someone_acting_on_behalf_of}}.

\begin{figure*}[!t]
\centering
\includegraphics[width=0.85\linewidth]{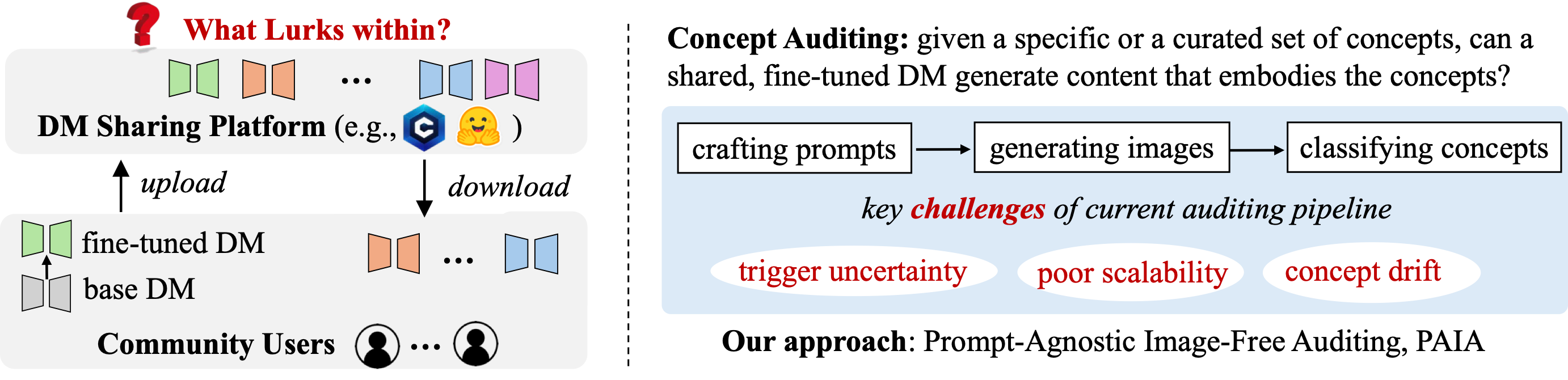}
\vspace{-1em}
\caption{Overview of \proposed: A Model-Centric Concept Auditing for Fine-Tuned Diffusion Models. Community users increasingly fine-tune and share DMs on platforms such as Civitai~\cite{civitai_website}, HuggingFace~\cite{huggingface_website}, and SeaArt~\cite{seaart_website}, introducing risks of unauthorized or sensitive content generation. Existing approaches rely on observable behaviors—prompts
and outputs—that are inherently unstable, easily manipulated, and costly to evaluate. We conduct a systematic study and propose a model-centric framework \proposed that bypasses these limitations by auditing internal model behavior directly. 
}
\label{fig:overview}
\vspace{-1em}
\end{figure*}

To mitigate such risks, some efforts have introduced proactive defenses, such as built-in safety filters designed to block harmful content by detecting unsafe prompts or screening generated images~\cite{Rando2022RedTeamingTS,Yang2023SneakyPromptER}. However, their effectiveness is limited. These filters are often implemented as optional add-ons and can be easily disabled after model download, leaving the core DM unrestricted. More importantly, these defenses suffer from a shared limitation: they rely on observable behaviors, prompts and outputs, that are unstable, easy to manipulate, and difficult to verify at scale. These limitations call for a deeper shift in auditing methodology: rather than focusing on what a model produces, we ask a more direct and scalable question: \textit{what lurks within}?

In this work, we study the problem of \textbf{concept auditing}: given a specific concept or a curated set of concepts, such as copyrighted characters, company logos, or celebrity identities, can a shared, fine-tuned diffusion model generate content that embodies those concepts? Rather than attempting to flag all forms of inappropriate content, we focus on concept-specific auditing at scale. This reflects real-world enforcement needs, where moderation and legal action often target clearly defined entities tied to intellectual property.
To keep the scope practical and aligned with real-world use cases, we focus on concepts that are visually distinctive and often subject to content moderation or IP protection. These include individual celebrity identities (\eg, ``Taylor Swift''), characters from well-known cartoons (\eg, ``Muppets''), and entities from games and movies. Our framework is designed to accommodate this diversity through an example-based approach: a concept is considered present if the model can generate recognizable outputs aligned with a small set of reference examples of the concept. This flexible definition allows us to support a broad range of concept granularities without relying on a rigid taxonomy. Details of how we define and structure concept scopes for auditing can be found in Section~\ref{sec:discussion_concept}.

A common approach to this problem is to audit models through their outputs: by crafting prompts that might trigger the concept of interest, generating images from the model, and using a classifier to determine whether the concept appears, as shown in Figure~\ref{fig:overview}. While this output-driven pipeline has become the dominant strategy, it relies on \textit{two fragile assumptions}: first, that effective prompts can be reliably discovered to generate the target concept; and second, that external detectors can accurately identify whether the generated output matches that concept.
In practice, both assumptions often fail for community-tuned models.
A key issue is \textbf{trigger uncertainty}: identifying prompts that consistently activate a target concept is inherently difficult due to the vast, discrete, and ambiguous nature of the prompt space. Optimization-based methods such as adversarial or reinforcement learning often converge on unnatural or semantically misaligned phrases, undermining their reliability.
Even when triggers are found, the assumption of accurate detection is weakened by \textbf{concept drift}~\cite{widmer1996learning,gama2014survey}: external classifiers or CLIP-based detectors are typically trained on natural images and often perform poorly on synthetic outputs, especially when prompt probing introduces distributional shifts that degrade both detection accuracy and optimization guidance~\cite{zhang2024generate,wen2024hard}. These problems are further compounded by \textbf{limited scalability}: discovering effective prompts and verifying them via image generation and classification is computationally intensive, often requiring hundreds of iterations. %

These limitations highlight a fundamental bottleneck in current auditing pipelines: their reliance on observable behaviors (\ie, prompts and their outputs) is inherently unstable, costly to evaluate, and easily manipulated. To overcome these issues, we propose a fundamentally different approach grounded in a shift of perspective: rather than auditing DMs by adjusting their inputs or analyzing their outputs, \textit{we treat the model itself as the source of truth}. This motivates a new direction in concept auditing: determining what a model has actually learned by examining its internal behavior, rather than what it happens to generate.

Our approach is built on two complementary strategies that collectively form the foundation of a model-centric auditing framework. Through theoretical analysis and empirical observation, we first uncover that \textit{the influence of prompts varies significantly throughout the denoising process}. Specifically, we find that prompts exert the strongest influence during the early stages of denoising, while their impact diminishes substantially in later steps. This insight motivates our first strategy: \textbf{a prompt-agnostic design} that aims to mitigate the effects of \textit{prompt uncertainty}. Rather than depending on fragile prompt optimization, we focus on the model's internal behavior during the later stages of generation, where its learned representations are more stable and less sensitive to the input prompt. This enables concept auditing that is inherently more robust to inaccurate, incomplete, or missing prompts.

While prompt-agnosticism reduces reliance on brittle input probing, current output-based detection methods remain limited by their dependence on image generation and external classifiers, leading to \textit{concept drift}, and are prohibitively \textit{expensive at scale}. To address this, we introduce our second strategy: \textbf{an image-free design} that further reinforces our model-centric paradigm. Instead of evaluating model outputs, we \textit{directly assess the model's internal denoising dynamics}. Specifically, we propose a metric called conditional calibrated error, which quantifies the behavioral discrepancy between a fine-tuned model and its corresponding base model when processing concept-relevant inputs. Notably, the base models are typically accessible in real-world deployments, because LoRA fine-tuning is designed to be modular, which requires the base model to be loaded alongside the LoRA weights at inference time. This makes the base model a natural and reliable reference point for behavioral comparison. By comparing internal activations directly, we can isolate fine-tuning effects and detect concept learning, without image sampling or reliance on noisy supervision.

Together, the proposed strategies offer two key advantages. First, they enhance \textbf{robustness} by grounding auditing in the model's own training dynamics, avoiding false positives caused by prompt misalignment or miscalibrated detectors. Second, they improve \textbf{efficiency} and \textbf{scalability} by operating entirely within the model's latent space—eliminating the need for prompt optimization, image generation, or downstream classification. Crucially, the prompt-agnostic and image-free designs are mutually reinforcing: by shifting the focus away from unstable input-output behavior and toward internal model representations, they enable a principled and scalable approach to concept auditing.
We operationalize this model-centric perspective in a unified framework: \textbf{Prompt-Agnostic Image-Free Auditing (\proposed)}, designed to efficiently and effectively determine whether a fine-tuned DM can generate a given target concept. 
In this paper, we assume the auditor has internal access to the uploaded DMs (parameters), but without access to proprietary fine-tuning data or user prompts, which aligns with standard auditing workflows on hosted model hubs.
To the best of our knowledge, \proposed represents the first practical, scalable, and systematically validated solution for concept auditing in fine-tuned DMs.
Our major contributions are summarized below.

\begin{itemize}[leftmargin=1em]
\item \textbf{A New Perspective on Concept Auditing.} We introduce Prompt-Agnostic Image-Free Auditing (\proposed), the first model-centric auditing framework that shifts the focus from observable inputs and outputs to the model's internal behavior. By treating the fine-tuned DM itself as the source of evidence, \proposed enables principled, robust, and scalable auditing.
\item \textbf{Prompt-Agnostic Design.} Adopting a model-centric perspective, we move beyond fragile prompt probing and examine how the model internally responds to prompts during generation. Our theoretical and empirical analysis reveals that prompt influence diminishes significantly in later denoising stages. Guided by this insight, we design a prompt-agnostic mechanism that analyzes late-stage model behavior to detect learned concepts, removing dependence on costly and ineffective prompt optimization.
\item \textbf{Image-Free Design.} Extending the model-centric paradigm, we shift from analyzing observable outputs to examining the model's internal behavior. We introduce an image-free detection mechanism based on a novel metric, conditional calibrated error, which captures behavioral deviations between a fine-tuned DM and its base counterpart. This approach enables accurate and scalable concept auditing without the need for image generation or external supervision.

\item \textbf{Extensive Experimental Validation in the Wild.} We conduct comprehensive experiments to evaluate \proposed on both controlled and real-world settings. In our controlled evaluation, we fine-tune $320$ DMs, each on a specific \textit{target concept}, where a concept refers to a recognizable visual entity such as an individual celebrity or a cartoon character. This dataset includes $50$ celebrity identities and $10$ cartoon characters. For real-world evaluation, we collect $771$ community-shared models from the Civitai platform, spanning $174$ celebrities, $145$ cartoon characters, $192$ videogame-related entities, and $179$ movie-based concepts, along with $81$ rare concepts. Across both settings, \proposed consistently achieves high accuracy (over $90$\%), efficiency ($18-40\times$ speedup), and remains robust under adaptive attacks, significantly outperforming existing baselines. To the best of our knowledge, this constitutes the first large-scale, systematic evaluation of concept auditing for fine-tuned DMs.

\end{itemize}

\section{Background}
\subsection{Diffusion Models (DMs)}
DMs have emerged as one of the most effective models for image generation  \cite{croitoru2023diffusion,ruiz2023dreambooth,zhang2023adding}, powering commercial image-generation applications, such as Stable Diffusion~\cite{rombach2021highresolution}, DALL-E 3~\cite{betker2023improving}, and MidJourney~\cite{borji2022generated}. 
Conceptually, the diffusion process can be described as a stochastic, iterative procedure in which noise is gradually introduced into an image until it becomes indistinguishable from pure noise. Specifically, in the diffusion process, given an image $\bm{x}_0$, a time step $0\le t \le T$, and a white noise vector $\vepsilon_t \sim \calN(0,\mI)$, a noisy image at time step $t$, $\vx_t$, is generated 
$
\vx_t = \sqrt{\overline{\alpha}_t} \vx_0 +  \sqrt{1-\overline{\alpha}_t}\vepsilon_t,
$
where $\overline{\alpha}_t$ represents a noise scheduling factor controlling the amount of noise injected at time step $t$. The noise schedule is designed to smoothly transform the image from the original $\vx_0$ to an almost pure noise distribution, $\vx_T$, as $t \to T$.

DMs are trained by learning to reverse the above forward diffusion process, progressively denoising $\vx_t$ to recover the original image $\vx_0$. %
During training, a neural network parameterized by $W$ learns to reverse this process by predicting $\vepsilon_t$ from the noisy input $\vx_t$, minimizing the MSE loss:
\begin{equation}
\min_{W} \E\left[\left\|{\vepsilon}_{W}(\vx_t) - \vepsilon_t \right\|^2\right].
\end{equation}
To reduce computational complexity for high-resolution generation, Latent Diffusion Models (LDMs)~\cite{rombach2021highresolution} operate in a lower-dimensional latent space $\vz_t$ and apply the same training objective.

\textbf{Text-to-Image (T2I) generation:}
T2I generation leverages textual input, \ie, prompts, to guide the image generation process~\cite{croitoru2023diffusion,ruiz2023dreambooth,zhang2023adding}. This begins by encoding the textual prompt into a prompt embedding \(\vp\) via a text encoder. The prompt embedding \(\vp\) is then fed into a DM \({\vepsilon}_{W}(\vz_t, \vp)\), conditioning the image generation process on the textual description. This integration is facilitated through a cross-attention mechanism. Specifically, text features are encoded into a key vector $K$ and a value vector $V$ using linear projection matrices $W_K$ and $W_V$, respectively. Simultaneously, the latent image features are projected into a query vector $Q$ using another projection matrix $W_Q$. A cross-attention map $S$ is computed as
\begin{equation}
\label{eq:attention_map}
S = \mathrm{softmax}\left(\frac{QK^T}{\sqrt{d}}\right), 
\end{equation}
where $d$ denotes the dimension of the projected vectors. The cross-attention map $S_{ij}$ represents the attention weights of the $j$-th tokens on the $i$-th pixel of the latent image features. Using this attention map, the output of the cross-attention layer is given by
\begin{equation}
\label{eq:attention_output}
Y=SV.
\end{equation}
By applying multiple cross-attention layers throughout the DM, the text features guide the iterative refinement of the latent image representation \(\vz_t\), ensuring that the generated image aligns with the input prompt.

The training process of the text-guided DM follows the same objective function as the original DM, where the prompt embedding \(\vp\) is used to minimize the denoising error,
\begin{equation}
\label{eq:denoising_error_prompt}
\min_{W} \E \left[ \left\|{\vepsilon}_{W}(\vz_t, \vp) - \vepsilon_t \right\|^2\right].
\end{equation}

During inference, the T2I generation process involves both conditional denoising (where the prompt embedding is set to \(\vp\)) and unconditional denoising (where the prompt embedding is encoded on a null character, i.e., \(\vp = \varnothing\)). To balance the impact of \(\vp\) on a generated image, the predicted noise at time step $t$ is calculated as
$
{\vepsilon}_{W}(\vz_t, \vp) = {\vepsilon}_{W}(\vz_t, \varnothing) + \eta \left( {\vepsilon}_{W}(\vz_t, \vp) - {\vepsilon}_{W}(\vz_t, \varnothing) \right),
$
where $\eta>1$ is the guidance scale that controls the strength of text conditioning. A higher $\eta$ increases the influence of the text prompt, making the generated image more closely aligned with the textual description, at the expense of diversity.
After the iterative denoising process, the final latent image representation \(\vz_0\) is decoded into a high-resolution image \(\vx_0\) using a pre-trained decoder.

\subsection{Parameter-Efficient Fine-Tuning (PEFT)}
PEFT has emerged as a widely adopted strategy for adapting large-scale models while significantly reducing computational and memory overhead~\cite{fu2023effectiveness,ding2023parameter,liu2022few}. This paper focuses primarily on Low-Rank Adaptation (LoRA) \cite{hu2021lora}, one of the most popular PEFT for fine-tuning DMs. LoRA introduces low-rank updates to pre-trained weight matrices during fine-tuning. Instead of updating the full parameter matrix, LoRA learns and stores a pair of low-rank matrices, $B \in \R^{d \times r}$ and $A \in \R^{r \times k}$, such that the original weight $W$ is modified as:
\begin{equation}
W' = W + \Delta W = W + BA,
\end{equation}
where the rank $r \ll \min(d, k)$ is typically small. This approach significantly reduces the number of trainable parameters, enabling efficient fine-tuning with minimal resource demands. In practice, LoRA is applied to all attention modules within the DM, including both self-attention and cross-attention layers. 

A major advantage of LoRA in real-world settings, such as community-driven model sharing, is its compact parameter footprint. Rather than sharing the entire set of fine-tuned model weights, users only need to share the small set of learned LoRA parameters. For example, when using rank $r=32$, the total size of LoRA parameters for Stable Diffusion 1.5 is approximately $25$ MB, orders of magnitude smaller than the full model size of around $5$ GB. This makes LoRA particularly attractive for scalable and lightweight distribution of fine-tuned models.

\section{\proposed: A Model-Centric Auditing Framework} \label{sec:method}

To address the practical challenges of concept auditing in the wild, we propose a model-centric framework, Prompt-Agnostic Image-Free Auditing (\proposed). Instead of relying on optimized prompts or generated outputs, \proposed analyzes the internal behavior of DMs to determine whether they can generate specific target concepts.

\proposed integrates two key innovations: a prompt-agnostic design that reduces reliance on prompt optimization by focusing on stable model behavior in the later stages of generation, and an image-free design that avoids output generation and external classifiers, enabling scalable and robust auditing. This section details each of these components, starting with the prompt-agnostic design, followed by the image-free design, and concludes with the full auditing pipeline.

\subsection{Threat Model}
\label{sec:threat}
We assume the auditor can access parameters from both the fine-tuned and the base model, but not know how the model is fine-tuned, \ie, what concepts are involved in the fine-tuning and what training data and prompts are used. This assumption is practical and common for model-sharing platforms (\eg, Civitai, HuggingFace). These platforms store model parameters for distribution and, in the event of DMCA takedown requests (\eg, for copyright infringement), can access both the fine-tuned and base models for analysis. Note that we don't assume a typical black-box setting (\eg, input–output access only), as it is not representative in concept auditing scenarios.
In addition, we assume the auditor has access to the sample images of the target concept, which will be used to evaluate whether the DMs can generate such concept.
This assumption is practical, \eg, in a takedown request for a cartoon character, the requester would provide sample images of the target character.

\subsection{Motivation and Challenges in the Wild}
While the widespread adoption of PEFT has enabled individual users to rapidly customize and share DMs across public platforms, it also introduces serious risks: many models are fine-tuned on unverified, proprietary, or sensitive datasets, with minimal oversight. This raises an urgent need for concept auditing: verifying whether a shared DM has learned to generate specific high-risk concepts.

In practice, auditing in this ecosystem is both essential and extremely challenging. Platforms such as Civitai allow users to upload and download models freely, often with incomplete, inconsistent, or misleading metadata. Trigger words may be omitted or obfuscated, and sample outputs are typically sparse, low-quality, or unavailable altogether. This auditing task remains unsolved at scale due to three core challenges: 

\begin{itemize}[leftmargin=1em]
    \item Prompt Uncertainty. Trigger phrases for sensitive concepts are rarely documented, often idiosyncratic, and embedded in a vast, ambiguous prompt space.
    \item Concept Drift in Detection. Output-based classifiers are typically trained on natural images and generalize poorly to synthetic content, especially under distribution shifts from prompt probing.
    \item Scalability Limitations. Existing pipelines require hundreds of iterations for prompt probing and image evaluation, making them impractical for large-scale auditing.
\end{itemize}

These challenges render the current auditing pipelines, whose reliance on observable behaviors, \ie, prompts and outputs, which is unfortunately brittle, expensive, and unreliable in real-world auditing scenarios. Therefore, we ask a more direct and scalable question: \textbf{what has the model internally learned, independent of its prompts or outputs?} This motivates the need for \textbf{a model-centric auditing framework} that avoids reliance on prompts or outputs and instead operates directly on the model's internal behavior.

\subsection{Prompt-Agnostic Design}
\label{sec:prompt_agnostic}

A core limitation of existing concept auditing pipelines is their dependence on observable behaviors—namely, carefully crafted input prompts and generated outputs. This reliance introduces inherent fragility: effective prompts are difficult to discover, model outputs are sensitive to slight input variations, and both are easily manipulated or obfuscated. These challenges hinder both the robustness and scalability of prompt-based auditing.

To address this, we adopt a fundamentally different perspective: rather than auditing the model through what it generates, we examine what it has learned. Specifically, we investigate whether concept auditing can be made prompt-agnostic by analyzing the model’s internal behavior during generation—thereby avoiding dependence on brittle prompt optimization.

\subsubsection{Motivation and Key Question}
Our approach is grounded in a critical question:
\textbf{Does prompt influence persist uniformly or diminish over time during the diffusion process?}
If the model’s reliance on prompt information fades in later denoising steps, it opens the door to auditing based on more stable, prompt-independent internal dynamics. To explore this, we analyze how the prompt embedding affects the model's behavior at different timesteps during generation.

\subsubsection{Theoretical Insight}

We begin with a theoretical analysis of the prompt's effect on the model's denoising process. In DMs, cross-attention modules align text and image features, making them key to understanding how prompts influence generation. Let $\vp$ be the text embedding extracted from the input prompt (Eq.~\ref{eq:denoising_error_prompt}). We derive the gradient of the cross-attention output with respect to the text embedding $\vp$:

\begin{lemma}
The gradient of the cross-attention function with respect to the text embedding $\vp$ is given by:
\begin{equation}
\label{eq:gradient}
\frac{\partial Y_i}{\partial \vp} =  (diag(S_i) - S_{i}S_i^T)(\frac{1}{\sqrt{d}}(XW_Q)W_K^T)\vp W_V + S_{i}W_V^T,\end{equation}
where $S$ denotes the cross-attention map (Eq.~\ref{eq:attention_map}), $S_i$ the $i$-th row of $S$, $X$ the image features calculated by the previous layers, and $Y_i$ the $i$-th output of the cross-attention layer (Eq.~\ref{eq:attention_output}). 
\end{lemma}
The full proof is available in Appendix (Section~\ref{sec:proof}). 

\begin{assumption}
\label{ass:embedding}
We assume that text embedding $\vp$ is compact and $\|\vp\| \le \vp^*$. 
\end{assumption}
ASSUMPTION \ref{ass:embedding} is satisfied in practice, as the text embedding P is normalized by the text encoder before being used in DMs. For example, Stable Diffusion uses a CLIP text encoder to extract the text embedding from the input prompt. The output embedding of CLIP is processed using layer normalization~\cite{lei2016layer}, defined as:
\begin{equation}
\text{LN}(\vx)= \frac{\vx-\mu(\vx)}{\sqrt{\sigma^2(\vx) + \epsilon}} \odot \boldsymbol\gamma + \boldsymbol\beta,
\end{equation}
where $\mu(\vx)= \frac{1}{D} \sum_{d=1}^D x_d$ and $\sigma^2(\vx)= \frac{1}{D}\sum_{d=1}^D (x_d - \mu(\vx))^2$ calculates the mean and variance of $x_d$ across layers. Here, $\boldsymbol{\gamma}, \boldsymbol{\beta} \in \mathbb{R}^D$ are learnable scaling factors.
Similarly, the input of each cross-attention layer is applied by layer normalization, which loosely bounds $\|\vp\|$ and ensures Assumption~\ref{ass:embedding} holds.

\begin{theorem}
\label{thm:cross_attn}
The cross-attention function is Lipschitz continuous with respect to $\vp$ under Assumption~\ref{ass:embedding}.
\end{theorem}
\begin{proof}
During inference, the parameters $W_Q$, $W_K$, and $W_V$ are fixed, and the targeted image features $X$ are also given and fixed. Furthermore, the softmax output satisfies $0 \leq S_{ij} \leq 1$. From Eq.~\ref{eq:gradient}, the gradient norm of $Y$ with respect to $\vp$ is bounded as:
\begin{align}
\label{eq:lipschitz}
\|\frac{\partial Y}{\partial \vp}\| &\leq C_1 \|diag(S_i) - S_{i}S_i^T\|\|\vp\| + C_2, \\
&\leq C_1 \vp^* + C_2, \nonumber
\end{align}
where $C_1 = \|W_Q (W_K X)^T W_V/\sqrt{d}\|$ and $C_2 = \|W_V\|$.
Since the cross-attention function is continuously differentiable and its gradient is bounded, it satisfies the conditions for Lipschitz continuity.
\end{proof}
\vspace{-0.5em}

\begin{figure}[!t]
\centering
\begin{subfigure}[h]{0.48\linewidth}
\centering
\includegraphics[width=\linewidth]{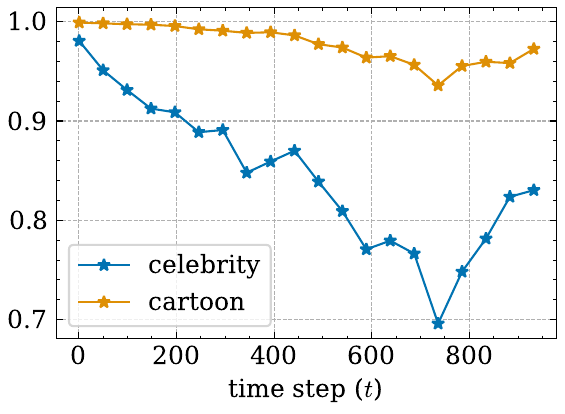}
\caption{$\|S\|$ over time $t$. }
\label{fig:attention_map_s}
\end{subfigure}
\begin{subfigure}[h]{0.48\linewidth}
\centering
\includegraphics[width=\linewidth]{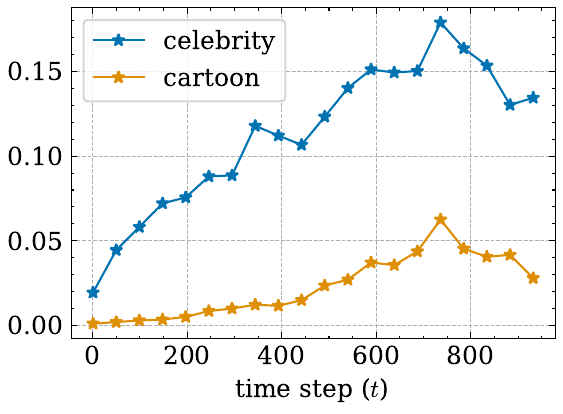}
\caption{$\|diag(S)-SS^T\|$ over time $t$. }
\label{fig:attention_curve}
\end{subfigure}
\vspace{-0.5em}
\caption{Impact of prompts on cross-attention. We present the results of attention map $\|S\|$ and $\|diag(S)-SS^T\|$ with respect to the first token [BOS]. The results are averaged over all celebrity and cartoon DMs, detailed in Section~\ref{sec:exp_setting}.}
\label{fig:attention_analysis}
\vspace{-1em}
\end{figure}

\begin{remark}
The gradient of the cross-attention function diminishes as the generation progresses (\ie, as $t$ becomes smaller).
\end{remark}

\subsubsection{Empirical Analysis}
\label{sec:method_prompt_emprical}
To complement this theoretical finding, we empirically examine how prompt influence evolves during the denoising process. Specifically, we track two metrics across timesteps: (1) the magnitude of the attention map $|S|$, and (2) the value of $|diag(S)-SS^T|$, which reflects prompt sensitivity. We focus on the [BOS] token, which typically carries the highest semantic weight in a prompt. 
As shown in Figure~\ref{fig:attention_analysis}, in the later stages of generation (smaller $t$), the attention map becomes more focused, with $|S|$ approaching $1$. Simultaneously, $|diag(S) - SS^T|$ decreases toward $0$, indicating a lower sensitivity to prompt variations.
This empirical trend confirms our theoretical result: in the later stages of the generation process, the model becomes increasingly confident about image contents and relies less on the conditioning prompt.

\begin{remark}
These findings indicate that the prompt's influence diminishes significantly in the later stages of diffusion. Hence, auditing the model during these stages can be done without relying on carefully engineered prompts.
\end{remark}

Note that Theorem~\ref{thm:cross_attn} applies to all token types, showing their influence decays during denoising. While our empirical studies focus on object-level nouns for practical reasons (\eg, legal needs for auditing), the theoretical result generalizes to all tokens (nouns, adjectives, and modifiers).

\subsection{Image-Free Design}
\label{sec:image_free}
While prompt-agnosticism addresses the fragility of input discovery, traditional output-based identification methods remain fundamentally limited. They depend on expensive image generation and are prone to concept drift, as external classifiers often misinterpret synthetic outputs. To overcome these issues, we introduce an image-free design that further reinforces our model-centric approach. \textit{Instead of evaluating what the model generates, we directly assess how it behaves during the denoising process}.

\subsubsection{Calibrated Error Measurement}
Our key insight is that the denoising errors of a fine-tuned DM differ from its base model when generating concepts the fine-tuned DM has learned. Specifically, if a fine-tuned DM can generate a new concept, the difference in denoising error between the fine-tuned DM and the base DM should be smaller for this concept compared to other, irrelevant concepts. This observation motivates the use of denoising errors as a reliable signal for concept auditing.

To quantify this, we introduce a \textbf{calibrated error measurement}, which captures the difference in denoising performance between the fine-tuned model's parameters \(W'\) and the base model's parameters \(W\). The calibrated error is defined as:

\begin{equation}
\label{eq:error_loss}
\mathcal{L}_{ce}^t \triangleq 
\E_{\vx\in \mathcal{D}_{target}} \left[\left\| {\vepsilon}_{W'}(\vz_t, \vp) - \vepsilon_0 \right\|^2 - \left\| {\vepsilon}_{W}(\vz_t, \vp) - \vepsilon_0 \right\|^2 \right],
\end{equation}
where $\mathcal{D}_{target}$ represents the images of the target concept.
Eq.~\ref{eq:error_loss} captures the signed difference. If the concept cannot be generated by the base model, its denoising loss is high. Fine-tuning reduces this loss, resulting in a negative $L_{ce}$ and indicating a stronger behavioral shift. Note that, given the prompt-agnostic design, the prompt $\vp$ in Eq.~\ref{eq:error_loss} is not required to be accurate.

\begin{figure}[!t]
\centering
\includegraphics[width=0.9\linewidth]{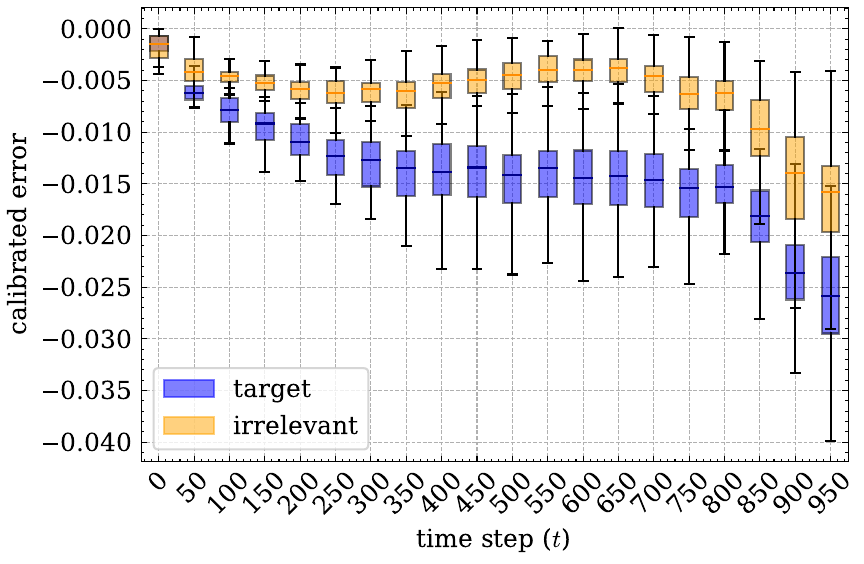}
\vspace{-0.5em}
\caption{Calibration error (CE) across denoising time steps. We compare the difference in denoising error between the fine-tuned model and its base model for target vs. irrelevant concepts. Early stages (large $t$) offer stronger separation but are more prompt-sensitive. This motivates our conditional calibrated error (CCE), which preserves early-stage signals while reducing prompt influence.}
\label{fig:loss_comparison}
\vspace{-0.5em}
\end{figure}

Using this metric, we evaluate whether a fine-tuned DM can generate a specific concept by comparing the calibrated error against a threshold $\tau$:
\begin{equation}
\label{eq:indicator}
\I(\mathcal{L}_{ce}^t < \tau),
\end{equation}
If the calibrated error falls below the threshold, the detector predicts that the fine-tuned DM can generate the target concept. Figure~\ref{fig:loss_comparison} demonstrates the effectiveness of this approach: the calibrated errors for target concepts are significantly lower than those for irrelevant concepts, enabling accurate differentiation.
Note that to ensure the analysis is not influenced by inaccurate prompts, we employ the original (accurate) prompts used during fine-tuning\footnote{The original prompts are used for diagnostic analysis only. All evaluations assume original prompts are unavailable.}. The calibration error is normalized by the denoising error of the base model to account for variations across time steps, ensuring a fair comparison.

\subsubsection{Conditional Calibrated Error Measurement}
While calibrated error effectively identifies target concepts, it exhibits varying performance across different stages of the DM generation process. 
As shown in Figure~\ref{fig:loss_comparison}, the calibration errors (using original prompts) in the early stages (large $t$) are more effective in concept identification.
This finding aligns with the recent studies on DM generation mechanisms~\cite{zhang2024cross,yi2024towards}: DMs primarily generate semantic information (\eg, structure) during early stages, while later stages focus on refining image details (\eg, texture). Since semantic information is more directly tied to concepts, details generated in later stages alone may lack sufficient distinction for effective concept auditing.

However, this presents a challenge for prompt-agnostic analysis. While the impact of prompts diminishes during the later stages, the calibrated errors at these stages are less effective for concept auditing. Conversely, early-stage calibrated errors are more effective but are influenced by prompts, creating a conflict between prompt-agnostic and image-free auditing designs.

To reconcile the tension between prompt-agnostic and image-free auditing, we propose a \textbf{conditional calibrated error (CCE)} that enables effective concept auditing across all generation stages.
Specifically, in the early stages, we calculate the calibrated error by freezing the original parameters of the cross-attention layers from the base model and apply fine-tuned parameters only to other layers (mainly self-attention layers). This setup reduces prompt influence during early denoising stages.
In the later stages, we use fine-tuned parameters across all layers to calculate the calibrated error, as the impact of prompts is already reduced.
The conditional calibrated error is formally defined as:
\begin{equation}
\label{eq:final_loss}
\mathcal{L}_{cce}^{t} \triangleq 
\begin{cases}
\E \left\| {\vepsilon}_{W'}(\vz_t, \vp) - \vepsilon_0 \right\|^2 - \left\| {\vepsilon}_{W}(\vz_t, \vp) - \vepsilon_0 \right\|^2, \quad t \le \gamma\\
\E \left\| {\vepsilon}_{W''}(\vz_t, \vp) - \vepsilon_0 \right\|^2 - \left\| {\vepsilon}_{W}(\vz_t, \vp) - \vepsilon_0 \right\|^2, \quad t > \gamma,
\end{cases}
\end{equation}
where $W''$ denotes the parameters excluding fine-tuned updates to cross-attention layers, and $\gamma$ is the cutoff time step separating early and later stages. In this work, $\gamma$ is set to the midpoint of the generation process: $\gamma = T/2$. 

By combining early-stage analysis with prompt-insensitive cross-attention and late-stage behavior with reduced prompt dependence, CCE enables robust and prompt-agnostic concept auditing across the entire generation process.

\subsection{Unsupervised Concept Detector}
In an ideal scenario, concept auditing could involve calculating the conditional calibrated error for a set of images that the fine-tuned DM can generate based on Eq.~\ref{eq:final_loss} and comparing these errors to a predefined threshold $\tau$ (Eq.~\ref{eq:indicator}). This threshold could be established by analyzing the error distributions of both target and irrelevant concepts, as shown in Figure~\ref{fig:loss_comparison}. 

However, in real-world scenarios, especially on community platforms such as Civitai~\cite{civitai_website}, the availability of generated images is often severely limited. For many fine-tuned models, users may upload fewer than $10$ example images, making it difficult or infeasible to learn a reliable threshold through supervised comparison between target and irrelevant concepts. 

To address this practical challenge, we introduce an \textbf{unsupervised concept detector} that eliminates the need for a large collection of target concept images. Instead of relying on positive examples, our framework leverages a set of irrelevant images, representing concepts that the DM is not expected to generate, and uses their conditional calibrated errors to train an outlier detection model. This model captures the typical error distribution associated with unlearned or unrelated concepts. At inference time, if the calibrated errors corresponding to a candidate concept deviate significantly from this baseline and are flagged as outliers, the detector concludes that the fine-tuned DM has likely learned to generate the target concept. 
Note that this approach requires only a small number of target concept images for inference, and no target examples are needed during training.
We implement the detector using Isolation Forest~\cite{liu2008isolation}, a widely used method for unsupervised outlier detection.

\begin{figure}[!t]
\centering
\includegraphics[width=\linewidth]{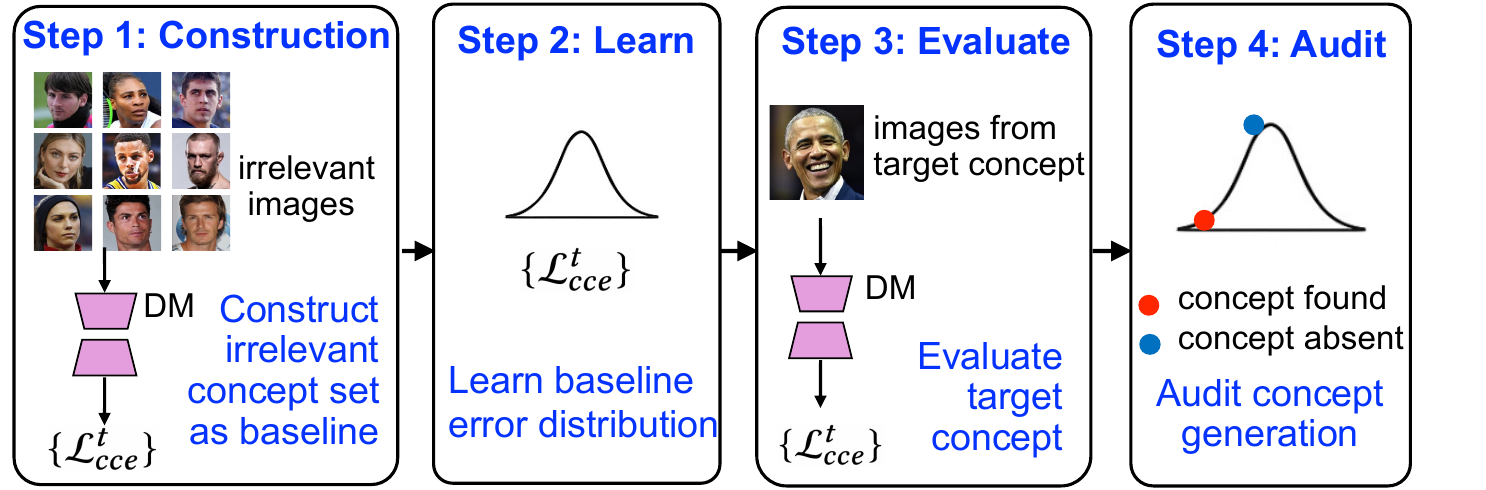}
\vspace{-1em}
\caption{Overview of \proposed Pipeline.}
\label{fig:pipeline}
\vspace{-1em}
\end{figure}

\subsection{Overall \proposed Pipeline}
We present the overall framework of \proposed, which integrates the prompt-agnostic and image-free design principles, along with unsupervised concept detector. 
The complete procedure of \proposed consists of the following four steps, as illustrated in Figure~\ref{fig:pipeline}.
A detailed algorithm is provided in the Appendix (Algorithm~\ref{alg:paia}).
\begin{itemize}[leftmargin=1em]
\item \textbf{Step 1: Construct Irrelevant Concept Set as Baseline.} 
Collect a set of images representing concepts the DM is not expected to generate. For each image, compute conditional calibrated error $\mathcal{L}_{cce}^t$ (Eq.~\ref{eq:final_loss}) across multiple time steps. CCE integrates early-stage measurements with frozen cross-attention parameters and late-stage measurements using all fine-tuned parameters.

\item \textbf{Step 2: Learn Baseline Error Distribution.} 
Train an outlier detection model using the constructed irrelevant concept set $\{\mathcal{L}_{cce}^t\}$. This model captures the typical internal behavior of unlearned concepts across all denoising stages.
\item \textbf{Step 3: Evaluate Target Concept.} 
Using a small number of images representing the target concept, compute their CCE values, again using only late-stage denoising, using the same procedure as in Step 1. 
\item \textbf{Step 4: Audit Concept Generation.} 
Apply the trained outlier detector to the target concept's calibrated errors. If these values are statistically distinct from the baseline (\ie, flagged as outliers), the model concludes that the DM has learned to generate the target concept.
\end{itemize}

\section{Experiments}
\subsection{Experimental Settings}
\label{sec:exp_setting}
We first evaluate the performance of \proposed on two categories of fine-tuned DMs and compare it against six baselines. 
Details on dataset construction, image examples, and random prompt term lists are provided in Appendix Section~\ref{sec:benchmark}.

\subsubsection{Datasets and Fine-tuned DMs}
We construct two benchmark datasets in the categories of \textit{celebrity} and \textit{cartoon}, each consisting of multiple visual concepts. In the celebrity category, we include $50$ individual celebrities, each represented by approximately $20$ images, sourced from the Celebrity-1000 dataset~\footnote{\url{https://huggingface.co/datasets/tonyassi/celebrity-1000}}. For cartoons, we collect $10$ distinct characters (e.g., Pikachu, Bart Simpson, Rick Sanchez) using publicly available datasets from HuggingFace\footnote{\url{https://huggingface.co/datasets}}, each containing on average $417$ images.

We generate textual prompts for each image using BLIP~\cite{li2022blip}. 
Datasets are split evenly into training and testing subsets; training data is used for LoRA fine-tuning, and test data is used for auditing evaluation.
The sample visualizations are included in Appendix Section~\ref{sec:benchmark}.

All DMs are fine-tuned from Stable Diffusion 1.5\footnote{\url{https://huggingface.co/stabilityai/stable-diffusion-1-5}} using Low-Rank Adaptation (LoRA), with rank $r = 64$, a learning rate of $1e$-$4$, and the AdamW optimizer. Each model is trained for $50$ epochs. For each category, we fine-tune models with $1$, $2$, and $3$ concepts per model to evaluate multi-concept auditing. This results in $50$ celebrity and $10$ cartoon models for each setting.

\subsubsection{\proposed Settings}
\label{sec:proposed_setting}
\noindent\textbf{Outlier detection model.}
In \proposed, we employ Isolation Forest as the detector, using the default implementation in~\cite{liu2008isolation} with its implicit automatic decision threshold (no additional calibration). We train the model on the irrelevant images, which cannot be generated by the DM model, and then predict if the images of the target concept are irrelevant. If not, then we predict the DM can generate this concept and vice versa.

The outlier detection model is trained with $100$ images with irrelevant concepts. In our evaluation, we randomly select images from the same category (celebrity or cartoon) but with different concepts. 
In practice, a straightforward strategy for constructing such a set is to use newly generated or low-popularity concepts, such as a lesser-known public figure or a recently created character.

\noindent\textbf{Prompt strategies.}
We assume we have no access to the original prompt with accurate trigger words for the target concept, we deploy three strategies to generate pseudo prompts that are fed into DMs for concept detection. The impact of different prompt strategies is investigated in Section~\ref{sec:prompt_generation}.
\begin{itemize}[leftmargin=1em]
    \item \textbf{Caption}. The prompt is synthesized using an image captioning model, GenerativeImage2Text~\cite{wang2022git}. To make a fair analysis, this image captioning model is selected to differ from the BLIP model used in data collection. The generated pseudo-prompts are then different from the ground-truth prompts, making it challenging for concept detection.
    In the evaluation, we use this strategy by default if not explicitly mentioned.

    \item \textbf{Random}. The prompt is composed of five randomly selected terms from a common word list for generating DMs, \eg, ``natural lighting,'' ``best quality,'' ``ultra detailed.'' 
    A complete list is presented in Table 10 in Appendix Section~\ref{sec:prompt_example}.

    \item \textbf{Null}. The null text, \ie, ``'', is employed as the prompt input.
\end{itemize}

\subsection{Baselines}
We consider baselines derived from two state-of-the-art prompt probing techniques and two image-based concept detectors.

\noindent\textbf{Prompt probing techniques:} We adopt two state-of-the-art prompt probing techniques and a naive probing approach using an image captioning method.
\begin{itemize}[leftmargin=1em]
    \item \textbf{MU}: UnlearnDiff~\cite{zhang2024generate} is a novel adversarial prompt generation method designed to evaluate the robustness of safety-driven unlearned DMs. UnlearnDiff optimizes the prompt via a variant of projected gradient descent (PGD) attack, Textgrad~\cite{houtextgrad}, which is tailored for discrete text optimization. UnlearnDiff leverages the denoising loss as the optimization objective, but it considers the denoising loss in the early generation stage (large $t$), which is more sensitive to prompts. Additionally, during the optimization, UnlearnDiff still needs an external classifier to determine whether the prompt can generate the image of the target concept. We follow their open-source implementation\footnote{\url{https://github.com/OPTML-Group/Diffusion-MU-Attack}}.
\item \textbf{PEZ}: Hard Prompts Made Easy (PEZ)~\cite{wen2024hard} optimizes hard prompts for text-to-image and text-only applications using a gradient-based discrete optimization technique. By iteratively projecting continuous embeddings onto discrete token spaces, PEZ balances the automation of gradient-based optimization with the interpretability and flexibility of hard prompts. We follow their open-source implementation\footnote{\url{https://github.com/YuxinWenRick/hard-prompts-made-easy}}.
\item \textbf{Naive}: We consider an image captioning model, GIT~\cite{wang2022git}, to derive prompts from the target images.
\end{itemize}

\noindent\textbf{Image-based concept identifiers:} We adopt two image-based concept identifiers.
\begin{itemize}[leftmargin=1em]
    \item \textbf{Image Classifier:} We train a ConvNeXt~\cite{liu2022convnet}-based image classifier to identify celebrity/cartoon concepts. The classifier is pre-trained on ImageNet and fine-tuned on the collected training data. After fine-tuning, the image classifier achieves $89.2$\% and $98.1$\% accuracy on celebrity and cartoon data, respectively.
    \item \textbf{CLIP:} This classifier utilizes the pre-trained CLIP~\cite{radford2021learning} model's capability to align text and image embeddings within a shared multimodal space. It performs classification by extracting an image’s embedding using the CLIP image encoder and comparing it with class label embeddings produced by the CLIP text encoder, assigning the image to the class with the highest cosine similarity. The CLIP classifier is widely adopted for classifying generated concepts due to its robustness and flexibility.
\end{itemize}

By combining three prompt probing techniques and two image-based concept identifiers, we consider six baselines in the evaluation: Naive Classifier, Naive CLIP, MU Classifier, MU CLIP, PEZ Classifier, and PEZ CLIP.

\subsection{Comparison with Baselines}
\label{sec:evaluation_baseline_comp}

\begin{table}[!tb]
\centering
\caption{Performance comparison on Celebrity DMs. We report accuracy, precision, recall, F1 score, and auditing time for auditing a concept using $10$ images.}
\label{tab:comparison_celebrity}
\resizebox{\linewidth}{!}{%
\begin{tabular}{@{}lrrrrr@{}}
\toprule
Detector & Accuracy & Precision & Recall & F1 Score & Time (s) \footnote{Left: average time per a curated set of concepts; Right: average time per single concept}\\ \midrule
Naive Classifier & 53\% & 100\% & 6\% & 12\% & 75.45 \\
Naive CLIP & 57\% & 64\% & 32\% & 43\% & 76.48 \\
MU Classifier & 58\% & 55\% & 92\% & 69\% & 2010.28 \\
MU CLIP & 51\% & 51\% & 98\% & 67\% & 1560.21 \\
PEZ Classifier & 75\% & 75\% & 76\% & 75\% & 1162.22 \\
PEZ CLIP & 59\% & 56\% & 80\% & 66\% & 933.32 \\ 
\proposed & 92\% & 90\% & 94\% & 92\% & 54.14/483.42 \\
\bottomrule
\end{tabular}%
}
\vspace{-1em}
\end{table}

\begin{table}[!tb]
\centering
\caption{Performance comparison on Cartoon DMs.}
\label{tab:comparison_cartoon}
\vspace{-0.5em}
\resizebox{\linewidth}{!}{%
\begin{tabular}{@{}lrrrrr@{}}
\toprule
Detector & Accuracy & Precision & Recall & F1 Score & Time (s) \\ \midrule
Naive Classifier & 79\% & 89\% & 66\% & 76\% & 76.54 \\
Naive CLIP & 63\% & 62\% & 68\% & 65\% & 77.20 \\
MU Classifier & 60\% & 56\% & 100\% & 71\% & 1005.20 \\
MU CLIP & 44\% & 47\% & 82\% & 59\% & 1113.03 \\
PEZ Classifier & 77\% & 68\% & 100\% & 81\% & 1053.47 \\
PEZ CLIP & 55\% & 54\% & 73\% & 62\% & 931.33 \\
\proposed & {92}\% & {86}\% & {100}\% & {93}\% & {56.55}/{485.30} \\ \bottomrule
\end{tabular}%
}
\vspace{-1em}

\end{table}

\vspace{.5em}
We compare \proposed with baselines and show the results in Tables~\ref{tab:comparison_celebrity} and~\ref{tab:comparison_cartoon}. 
\proposed consistently achieves the highest accuracy (92\% and 91\%), precision (90\% and 89\%), and recall (94\% for both categories), demonstrating its superior performance in concept auditing.
Among the baselines, the PEZ Classifier performs relatively well with accuracy of 75\% and 77\%, but still significantly lower than \proposed.
Other methods, such as MU Classifier and MU CLIP, achieve high recall but suffer from poor precision and accuracy, indicating a high positive rate.

Auditing times (in seconds) for analyzing a concept with $10$ images are also reported in the tables.
For \proposed, we evaluate auditing time in two scenarios: (1) analyzing a single concept for a fine-tuned DM and (2) analyzing multiple concepts, which is more representative of practical applications.
In the second scenario, \proposed achieves significantly reduced auditing time, as the conditional calibrated errors for irrelevant images (Step 2) are computed only once and reused for subsequent concepts.

Compared to baselines, \proposed demonstrates much greater efficiency and scalability. Prompt probing techniques (MU and PEZ) often require considerably longer optimization times. Even against naive image-captioning-based methods, \proposed demonstrates slightly higher efficiency in the second scenario since it computes denoising errors at a limited number of time steps rather than performing the full generation process.

\subsection{Ablation Study}
This section systematically analyzes the key factors that may influence the performance of \proposed. 

\label{sec:prompt_generation}
\begin{figure}[!t]
\centering
\begin{subfigure}[h]{0.48\linewidth}
\centering
\includegraphics[width=0.95\linewidth]{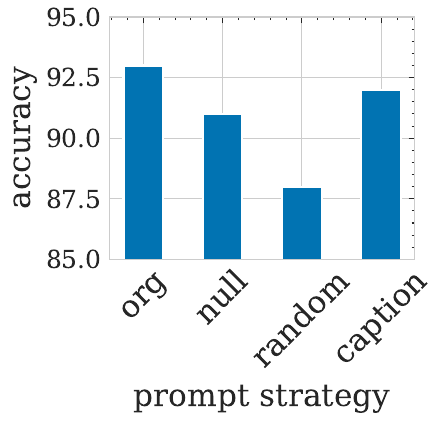}
\caption{Celebrity DMs}
\end{subfigure}
\hfill
\begin{subfigure}[h]{0.48\linewidth}
\centering
\includegraphics[width=0.95\linewidth]{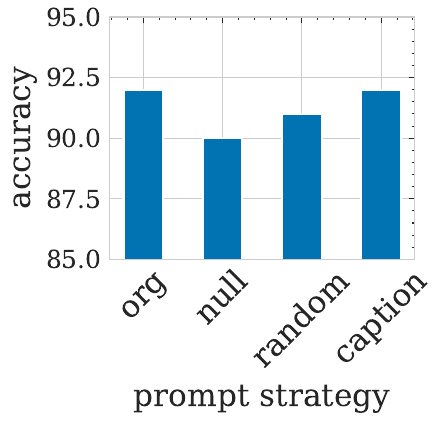}
\caption{Cartoon DMs}
\end{subfigure}
\vspace{-0.5em}
\caption{Auditing performance of \proposed with different prompt generation strategies.}
\label{fig:prompt_type}
\vspace{-1em}
\end{figure}

\vspace{.5em}
\noindent\textbf{Impact of Prompt Generation Methods:}
We evaluate \proposed using different prompt generation strategies, including Null, Random, and Caption (as described in Section~\ref{sec:proposed_setting}) and compare their performance with \proposed using the original (accurate) prompts (org). As shown in Figure~\ref{fig:prompt_type}, the Caption strategy achieves nearly identical performance to the original prompts, while the Null and Random strategies exhibit only slightly lower performance. This demonstrates the robustness of \proposed's prompt-agnostic design, ensuring effective auditing even with less accurate prompts.

\vspace{.5em}
\noindent\textbf{Effectiveness of Conditional Calibrated Error Measurement:}
We analyze the effectiveness of conditional calibrated error measurement (Eq.~\ref{eq:final_loss}). 
We compare the auditing performance across four settings: 1) freezing cross-attention layers at the early stages (\ie, w/ conditional calibration), 2) not freezing any layers (\ie, w/o conditional calibration), 3) freezing cross-attention layers at all stages (\ie, w/ self-attention only), and 4) freezing self-attention layers at all stages (\ie, w/ cross-attention only).
As shown in Figures~\ref{fig:selective_celebrity} and~\ref{fig:selective_cartoon}, by freezing parameters in cross-attention layers at the early stages of generation, conditional calibration error improves the performance of \proposed in most cases.

\begin{figure}[!t]
\centering
\includegraphics[width=0.7\linewidth]{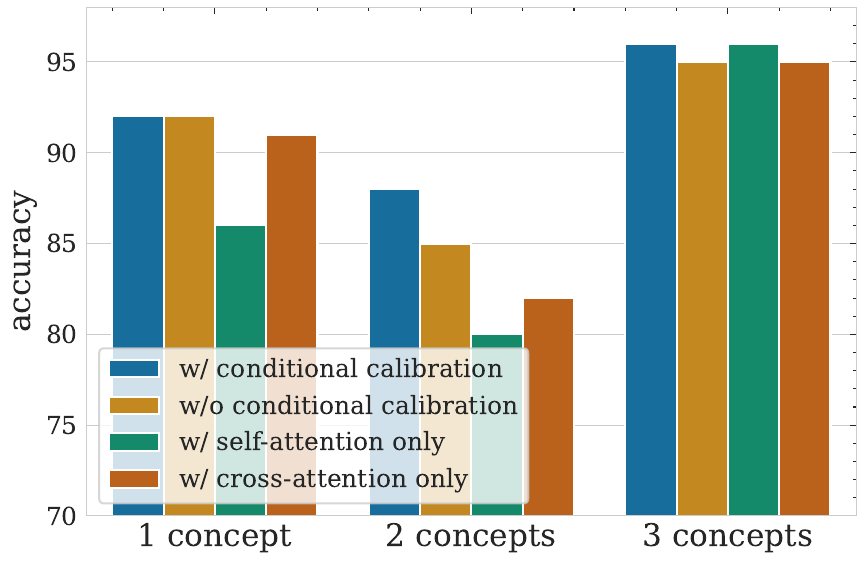}
\vspace{-0.5em}
\caption{Effectiveness of conditional calibration error measurement on Celebrity DMs.}
\label{fig:selective_celebrity}
\vspace{-1em}
\end{figure}

\begin{figure}[!t]
\centering
\includegraphics[width=0.7\linewidth]{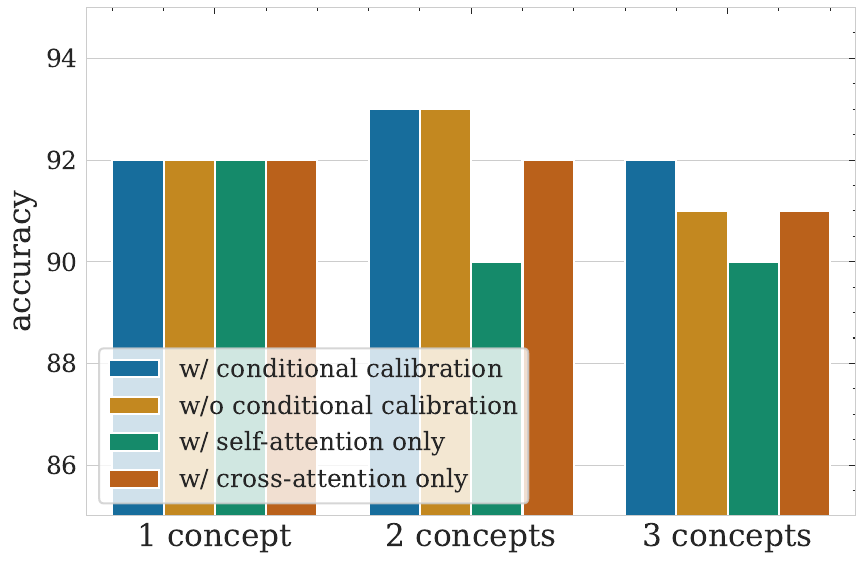}
\vspace{-0.5em}
\caption{Effectiveness of conditional calibration error measurement  on Cartoon DMs.}
\label{fig:selective_cartoon}
\vspace{-1em}
\end{figure}

\vspace{.5em}
\noindent\textbf{Effectiveness of Outlier Detection Algorithms:}
We investigate the effectiveness of different outlier detection methods, including Isolation Forest (IF)~\cite{liu2008isolation}, Angle-based outlier detection (ABOD)~\cite{kriegel2008angle}, k-Nearest Neighbors (kNN)~\cite{laaksonen1996classification}, Gaussian Mixture Model (GMM)~\cite{aggarwal2017introduction}, and One-class SVM (OneSVM)~\cite{manevitz2001one}. We evaluate these algorithms on both celebrity and cartoon DMs with 1 concept. Among the algorithms investigated, Isolation Forest (the default algorithm in \proposed) consistently achieved the highest accuracy and F1 scores, but other robust outlier detection algorithms like ABOD and kNN also performed competitively, with slightly lower precision.

\begin{table}[!tb]
\centering
\caption{Effectiveness of outlier detection algorithms on Celebrity DMs.}
\label{fig:outlier_celebrity}
\vspace{-0.5em}
\resizebox{0.9\linewidth}{!}{%
\begin{tabular}{@{}lrrrr@{}}
\toprule
Algorithm & Accuracy & Precision & Recall & F1 Score \\ \midrule
IF & 92\% & 90\% & 94\% & 92\% \\
ABOD & 90\% & 84\% & 98\% & 91\% \\
kNN & 89\% & 83\% & 98\% & 90\% \\
GMM & 53\% & 52\% & 100\% & 68\% \\
OneSVM & 64\% & 58\% & 100\% & 74\% \\ \bottomrule
\end{tabular}%
}
\vspace{-0.5em}
\end{table}

\begin{table}[!tb]
\centering
\caption{Effectiveness of outlier detection algorithms on Cartoon DMs.}
\label{fig:outlier_cartoon}
\vspace{-0.5em}
\resizebox{0.9\linewidth}{!}{%
\begin{tabular}{@{}lrrrr@{}}
\toprule
Algorithm & Accuracy & Precision & Recall & F1 Score \\ \midrule
IF & 92\% & 86\% & 100\% & 93\% \\
ABOD & 88\% & 81\% & 100\% & 89\% \\
kNN & 89\% & 82\% & 100\% & 90\% \\
GMM & 57\% & 54\% & 100\% & 70\% \\
OneSVM & 73\% & 65\% & 100\% & 79\% \\ \bottomrule
\end{tabular}%
}
\vspace{-0.5em}
\end{table}

\vspace{.5em}
\noindent\textbf{Performance on Individual Time Steps:}
We investigate the performance of \proposed when predicting based on individual time steps during the generation process. As shown in Figure~\ref{fig:timestep}, \proposed performs best in the middle stages of generation. This observation aligns with our previous analysis: the prompt-agnostic design of \proposed is highly effective in the later stages, while the image-free design excels in the early stages. However, relying on individual time steps may not fully utilize the potential of \proposed, which analyzes across all time steps. This highlights the effectiveness of the conditional calibration error, which combines the strengths of both designs across all stages to deliver robust predictions.
\begin{figure}[!tb]
\centering
\includegraphics[width=0.5\linewidth]{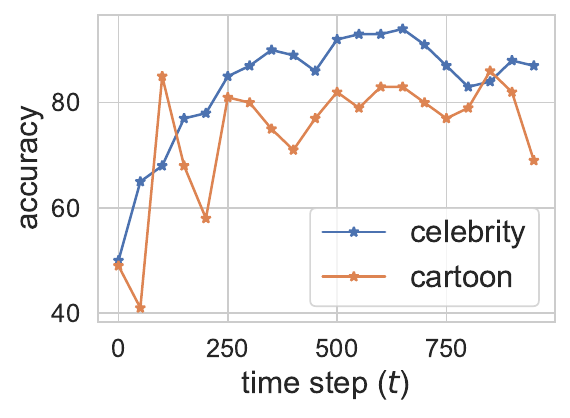}
\vspace{-0.5em}
\caption{Performance on different time steps.}
\label{fig:timestep}
\vspace{-0.5em}
\end{figure}

\begin{figure}[!t]
\centering
\begin{subfigure}[h]{0.45\linewidth}
\centering
\includegraphics[width=0.9\linewidth]{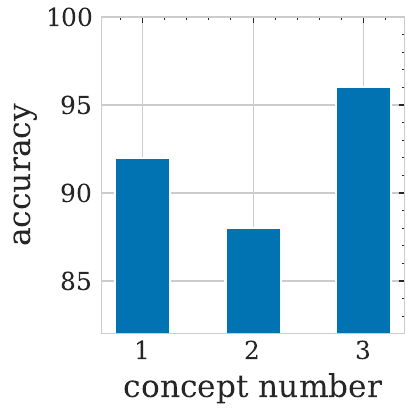}
\caption{Celebrity DMs}
\label{fig:concept_num_celebrity}
\end{subfigure}
\hfill
\begin{subfigure}[h]{0.45\linewidth}
\centering
\includegraphics[width=0.9\linewidth]{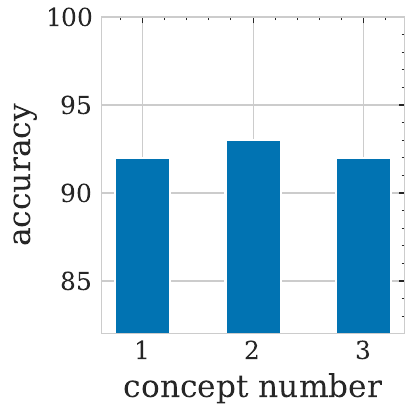}
\caption{Cartoon DMs}
\label{fig:concept_num_cartoon}
\end{subfigure}
\vspace{-0.5em}
\caption{Impact of concept number in each fine-tuned DM.}
\label{fig:impact_concept_number}
\vspace{-0.5em}
\end{figure}

\vspace{.5em}
\noindent\textbf{Impact of Number of Concepts:}
We investigate how the number of concepts in fine-tuned DMs affects auditing performance. As shown in Figure~\ref{fig:impact_concept_number}, there is no significant difference in performance when more concepts are included, except for a slight degradation when fine-tuning DMs on two celebrity concepts. This suggests that \proposed remains robust across a range of concept complexities in fine-tuned DMs.

\vspace{.5em}
\noindent\textbf{Prompt Sensitivity Analysis:} To quantify the prompt sensitivity in early stages, we compared auditing accuracy across three prompt strategies: caption, random, and null, where the ``caption'' strategy is most similar to the original prompt, compared to ``random'' and ``null'' strategies. As shown in Figure~\ref{fig:prompt_sens}, auditing performance is more sensitive to prompts in early stages (large $t$), as caption prompts significantly outperform random and null prompts. In contrast, accuracy degrades at later stages (small $t$) for all prompts, where calibrated errors become less effective in auditing. These findings motivate our CCE design, which combines prompt-aware early-stage and prompt-agnostic late-stage signals. 

\begin{figure}[!t]
\centering
\begin{subfigure}[h]{0.48\linewidth}
\centering
\includegraphics[width=\linewidth]{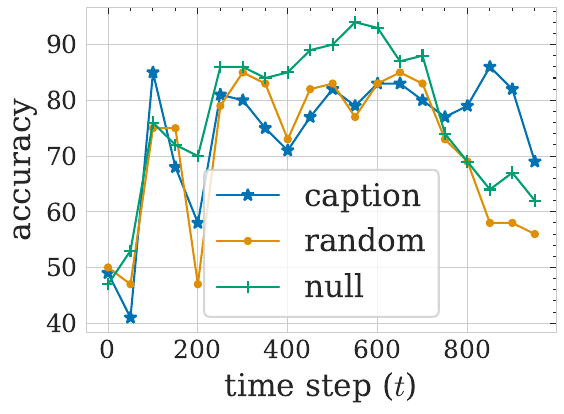}
\caption{celebrity}
\end{subfigure}
\begin{subfigure}[h]{0.48\linewidth}
\centering
\includegraphics[width=\linewidth]{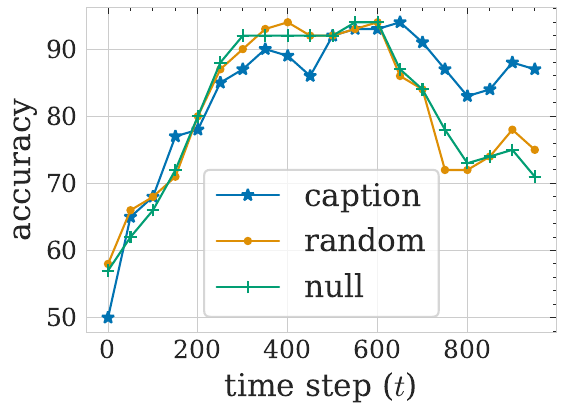}
\caption{cartoon}
\end{subfigure}
\caption{Prompt sensitivity analysis in different stages.}
\vspace{-0.5em}
\label{fig:prompt_sens}
\vspace{-1em}
\end{figure}

\subsection{Performance of \proposed Against Adaptive Attacks}
\label{sec:eval_adaptive}
We evaluate the robustness of \proposed under adaptive attack scenarios where attackers have full knowledge of the auditing framework and intentionally modify the fine-tuning procedure to evade detection. Specifically, we design five adaptive strategies that explicitly target the threat model and the core design of \proposed. The adaptive attacks are described below.
\begin{itemize}[leftmargin=1em]
    \item \textbf{Prompt Deviation Attack:} The attacker intentionally fine-tunes DMs using mismatched prompts (e.g., using ``Tom Hanks'' as the prompt for a model fine-tuned on ``Brad Pitt'' images). We train $50$ celebrity and $10$ cartoon models with deviated prompts to evaluate \proposed's robustness.
    \item \textbf{Limited Calibration Attack:} \proposed relies on estimating calibration errors on irrelevant data to train its outlier detector. In this attack, the auditor is restricted to a smaller irrelevant dataset, reducing the accuracy of calibration and weakening detection. 
    \item \textbf{Regularization Attack:} The attacker adds the calibrated error (Eq.~\ref{eq:error_loss}) as a regularization term during fine-tuning. This attack aims to explicitly suppress the behavioral difference between the fine-tuned model and the base model, thereby reducing the discriminative power of \proposed and bypassing the detection.
    \item \textbf{Early-Freezing Attack:} In this attack, the fine-tuning is restricted to the late stages (\ie, $t < \gamma$), leaving the DMs frozen in the early stages. Since early stages are more effective at revealing fine-tuning-induced behavioral shifts, by freezing the DMs during this stage, this attack aims to minimize the model changes detectable \proposed while still preserving the generation performance (by finetuning in the late stage).
    \item \textbf{Late-Freezing Attack:} DMs are fine-tuned only during early stages ($t > \gamma$), with late (prompt-sensitive) stages frozen. By avoiding updates to the stages most influenced by specific prompts, this attack aims to keep the DMs prompt-agnostic, not limiting to the specific prompts in fine-tuning, thus reducing prompt-induced divergence between the fine-tuned and base models.
\end{itemize}

These adaptive attacks are specifically designed to exploit the key ``white-box'' assumption. In this setting, the auditor knows the model parameters, but not the fine-tuning process. Accordingly, the adaptive attacks could modify aspects of fine-tuning process hidden from the auditor: Prompt Deviation Attack alters the prompt in the training data; Regularization Attack modifies the training loss by regularizing the calibration error; Early-Freezing Attack and Late-Freezing Attack change which parameters are updated during fine-tuning. By targeting unseen parts of the pipeline, these attacks represent realistic adaptive strategies under practical scenarios.

\vspace{0.5em}
\noindent\textbf{Performance against Prompt Deviation Attack:} As shown in Tables~\ref{tab:comparison_abnormal_celebrity} and~\ref{tab:comparison_abnormal_cartoon}, \proposed remains unaffected by Prompt Deviation Attack, achieving a comparable performance to those with normal prompts (Tables~\ref{tab:comparison_celebrity} and~\ref{tab:comparison_cartoon}). This is mainly due to our prompt-agnostic design. Across both datasets, \proposed consistently outperforms all baseline methods, achieving the highest accuracy, precision, recall, and F1 score. PEZ Classifier shows competitive performance across baselines, while methods like MU Classifier and MU CLIP continue to suffer from high false positive rates, leading to poor precision and low overall accuracy. Interestingly, these methods perform even worse than the naive prompt probing approaches, such as Naive Classifier and Naive CLIP, which do not rely on sophisticated optimization. This suggests that the MU methods face significant challenges when attempting to optimize for deviated prompts.
The results underscore \proposed's robustness against Prompt Deviation Attack.

We also observe that auditing performance under Prompt Deviation Attack is often better than under normal prompts. We attribute this to differences in the fine-tuned DMs rather than the detectors themselves, likely due to stronger memorization of deviated prompts in fine-tuned DMs, a phenomenon noted in backdoor attack research. For instance, Naive Classifier and Naive CLIP demonstrate improved auditing performance on deviated prompts, despite not utilizing any prompt information. This suggests that the difference arises from the performance of the fine-tuned DMs rather than the concept auditing methods.

\begin{table}[!tb]
\centering
\caption{Performance comparison under Prompt Deviation Attack (Celebrity DMs). \proposed outperforms baseline detectors, while achieving comparable performance to those without adaptive attacks in Table~\ref{tab:comparison_celebrity}.}
\label{tab:comparison_abnormal_celebrity}
\vspace{-0.5em}
\resizebox{0.9\linewidth}{!}{%
\begin{tabular}{@{}lrrrr@{}}
\toprule
Detector & Accuracy & Precision & Recall & F1 Score \\ \midrule
Naive Classifier & 75\% & 86\% & 60\% & 71\% \\
Naive CLIP & 77\% & 81\% & 70\% & 75\% \\
MU Classifier & 57\% & 54\% & 100\% & 70\% \\
MU CLIP & 57\% & 54\% & 100\% & 70\% \\
PEZ Classifier & 82\% & 81\% & 84\% & 82\% \\
PEZ CLIP & 65\% & 60\% & 92\% & 72\% \\
\textbf{\proposed} & 94\% & 89\% & 100\% & 94\% \\ \bottomrule
\end{tabular}%
}
\vspace{-1em}
\end{table}

\begin{table}[!tb]
\centering
\caption{Performance comparison under Prompt Deviation Attack (Cartoon DMs).
\proposed outperforms baseline detectors, while achieving comparable performance to those without adaptive attacks in Table~\ref{tab:comparison_cartoon}.
}
\label{tab:comparison_abnormal_cartoon}
\vspace{-0.5em}
\resizebox{0.9\linewidth}{!}{%
\begin{tabular}{@{}lrrrr@{}}
\toprule
Detector & Accuracy & Precision & Recall & F1 Score \\ \midrule
Naive Classifier & 85\% & 86\% & 84\% & 85\% \\
Naive CLIP & 55\% & 55\% & 58\% & 56\% \\
MU Classifier & 56\% & 53\% & 100\% & 69\% \\
MU CLIP & 45\% & 47\% & 86\% & 61\% \\
PEZ Classifier & 83\% & 75\% & 100\% & 85\% \\
PEZ CLIP & 53\% & 52\% & 72\% & 61\% \\
\textbf{\proposed} & 92\% & 86\% & 100\% & 93\% \\ \bottomrule
\end{tabular}%
}
\end{table}

\vspace{.5em}
\noindent\textbf{Performance against Limited Calibration Attack:} 
\proposed relies on accurate estimation of the conditional calibration error distributions for irrelevant images to train the outlier detector. 
To evaluate \proposed's robustness against Limited Calibration Attack, we vary the number of irrelevant images used. As shown in Figure~\ref{fig:num_calibration}, $120$ images are sufficient for reliable concept auditing. This number is realistic since irrelevant examples (\eg, less popular celebrities or new cartoon characters) are typically easy for platforms to collect. Adding more images beyond this threshold yields little improvement in accuracy while increasing auditing time.

\begin{figure}[!t]
\centering
\includegraphics[width=0.5\linewidth]{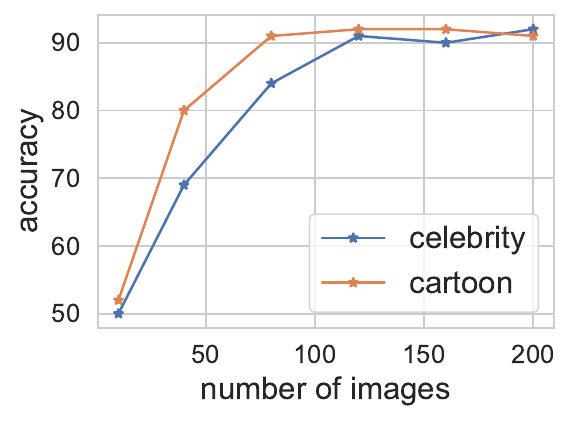}
\vspace{-1em}
\caption{Robustness of \proposed against Limited Calibration Attack. We evaluate \proposed's performance by limiting the number of irrelevant images used in the outlier detector. The default number of irrelevant images is $200$ (no attack).}
\label{fig:num_calibration}
\vspace{-0.5em}

\end{figure}

\begin{table}[!tb]
\caption{Robustness of \proposed against Regularization, Late-Freezing, and Early-Freezing Attacks. Auditing accuracy is reported on Celebrity and Cartoon DMs.}
\label{tab:robust_adaptive}
\vspace{-0.5em}
\resizebox{0.9\linewidth}{!}{%
\begin{tabular}{@{}lrr@{}}
\toprule
Adaptive Attacks & Celebrity DMs & Cartoon DMs\\ \midrule
No Attack      & 92\%   & 92\% \\
Regularization Attack       & 90\%   & 92\% \\
Late-Freezing Attack  & 94\%   & 90\% \\
Early-Freezing Attack & 93\%   & 91\% \\ \bottomrule
\end{tabular}%
}
\end{table}

\vspace{.5em}
\noindent\textbf{Performance Against Regularization, Late-Freezing, and Early-Freezing Attacks:} 
As shown in Table~\ref{tab:robust_adaptive}, \proposed consistently maintains high auditing accuracy across all three adaptive attacks. Specifically, the performance drops only marginally (within 2\%) for all the attacks, demonstrating strong robustness of \proposed. The Regularization Attack and Late-Freezing Attack show some effectiveness in the celebrity and cartoon categories, respectively, while the Early-Freezing Attack is ineffective in both. 
We further analyzed the results. In the Regularization Attack, we observe that, the calibrated error (Eq.~\ref{eq:error_loss}) is significantly reduced by the attack and becomes much smaller compared with the denoising error (Eq.~\ref{eq:denoising_error_prompt}). However, \proposed remains effective even when the behavioral difference becomes subtle.
The ineffectiveness of the Early-Freezing Attack is due to \proposed's prompt-agnostic design, which does not rely on specific training prompts and thus resists prompt-diversity obfuscation.
Finally, we believe both the Early- and Late-Freezing Attacks fail to bypass \proposed due to the flexibility of the Conditional Calibrated Error, which captures behavioral shifts across both early and late stages.

\section{Evaluation on Real-World DMs.}
In this section, we evaluate \proposed on real-world DMs on Civitai, a popular online DM-sharing platform~\cite{civitai_website}.

\subsection{Statistic Analysis on Civitai}
We first collect and analyze metadata from $100,200$ LoRA DMs available on Civitai at the time of collection.
Our analysis reveals that most DMs do not specify trigger words in their metadata (Figure~\ref{fig:civitai_trigger}). In some cases, ``unclaimed'' prompts like ``ohwx,'' ``sks'' are used. These prompts, which are not part of the model vocabulary, are often introduced during fine-tuning to trigger unique concepts. However, the absence of trigger words or the use of uncommon words creates significant challenges for concept auditing, as the intended concepts cannot be easily inferred from the metadata.

To perform a comprehensive evaluation, we conduct an analysis on the models hosted by Civitai. Based on the analysis result (Appendix Section~\ref{sec:app_civitai}), we identify four major topics of concepts: ``celebrity,'' ``game,'' ``cartoon,'' and ``movie,'' and a dominant DM: Stable Diffusion 1.5 (SD1.5).
Hence, we download and analyze the LoRA models from these four categories and with Stable Diffusion 1.5 as the base model.

\begin{figure}[!t]
\centering
\includegraphics[width=\linewidth]{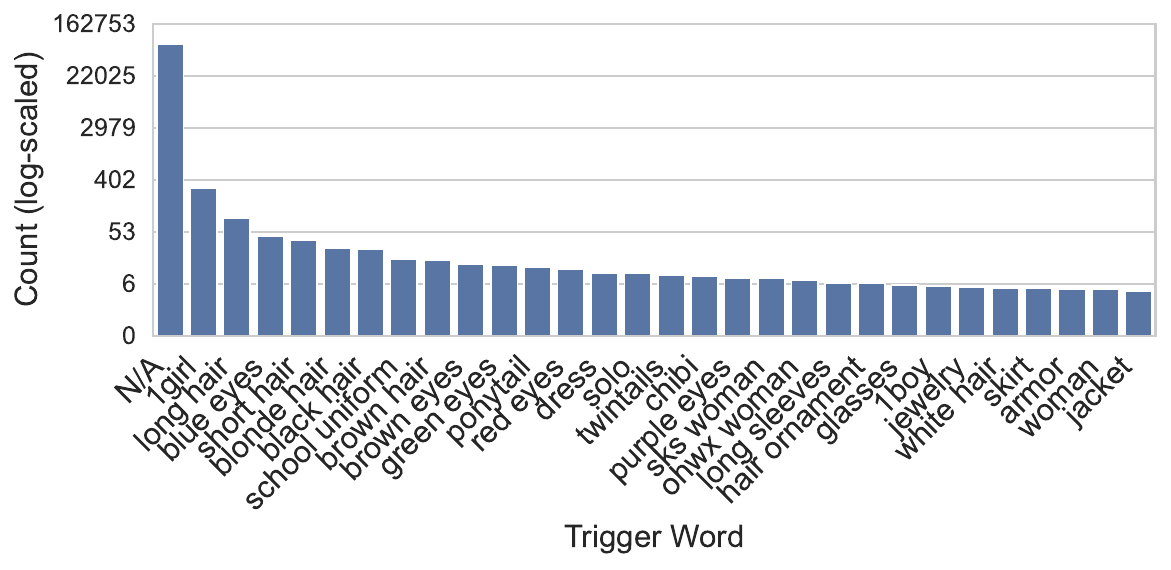}
\vspace{-1.5em}
\caption{Frequency of common trigger words on Civitai. N/A indicates that the user does not provide trigger words.}
\label{fig:civitai_trigger}
\vspace{-1em}
\end{figure}

\subsection{Model Collection}
The model collection process is as follows. First, we search for the LoRA models using the tags provided by each model in Civitai. For the celebrity topic, we use the tag ``celebrity,'' ``actress,'' ``real\_person,'' and ``actor.'' For the cartoon topic, we use the tag ``cartoon,'' ``manga,'' and ``anime.'' For the videogame topic, we use the tag of ``videogame,'' ``videogames,'' and ``video\_game.'' For the movie topic, we use the tag of ``movie.''
Second, we download the most frequently downloaded LoRA models in the search results. 
Third, since the model files may stored in different formats (\eg, different parameter names), we standardize all DMs to a consistent format before performing analysis. 
Lastly, we manually remove DMs with duplicate concepts (e.g., multiple DMs fine-tuned on the same celebrity or character) and exclude style-related LoRA models that are not concept-specific.
After filtering, we have collected $174$ celebrity models, $145$ cartoon models, $192$ videogame models, and $179$ movie models.

\subsection{Auditing Performance of \proposed}

\begin{table}[!tb]
\centering
\caption{Auditing performance of \proposed on Civitai DMs. }
\label{tab:detection_civitai}
\vspace{-0.5em}
\begin{tabular}{@{}lrrrr@{}}
\toprule
Category & Accuracy & Precision & Recall & F1 Score \\ \midrule
Celebrity & 93.97\% & 94.74\% & 93.10\% & 93.91\% \\
Cartoon & 96.90\% & 97.22\% & 96.55\% & 96.89\% \\
Videogame & 95.83\% & 94.44\% & 97.40\% & 95.90\% \\
Movie & 92.74\% & 91.80\% & 93.85\% & 92.82\% \\ \bottomrule
\end{tabular}%
\vspace{-0.5em}
\end{table}

We evaluate \proposed on the collected LoRA models across the four categories. For evaluation, we use 10 positive samples (images associated with the target concept) and 10 negative samples (images unrelated to the target concept) for each model. Table~\ref{tab:detection_civitai} summarizes the auditing performance of \proposed.
\proposed achieves high auditing accuracy, precision, and recall across all categories. Specifically, the cartoon category achieves the best performance with 97\% across all metrics, followed closely by the videogame and celebrity categories. The movie category shows slightly lower precision but maintains robust performance overall. These results demonstrate the effectiveness of \proposed in detecting diverse concepts across real-world DMs.

We further analyze the failure cases in the auditing results and observe that a significant portion of these failures arise from concepts that are closely related. Specifically, these include different concepts generated by the same creator or those originating from the same movie or TV show (\eg, Star Wars). Such scenarios introduce shared underlying features in fine-tuned DMs, making it challenging to distinguish between fine-grained variations. We will investigate this challenging scenario and enhance \proposed's fine-grained auditing capabilities in our future work.

\subsection{Auditing Performance on Rare Concepts}
\label{sec:eval_civitai_rare}
\begin{table}[!tb]
\caption{Auditing performance of \proposed on Civitai DMs with rare concepts.}
\label{tab:civitai_rare}
\vspace{-0.5em}
\begin{tabular}{lrr}
\hline
Category  & Common concepts & Rare concepts \\ \hline
Celebrity & 93.97\%         & 96.43\%       \\
Cartoon   & 96.90\%         & 93.75\%       \\
Videogame & 95.83\%         & 95.00\%       \\
Movie     & 92.74\%         & 96.15\%       \\ \hline
\end{tabular}%
\vspace{-1em}
\end{table}

Previous analysis explores \proposed's performance on the common concepts with most downloads. For rigorous evaluation, we further assess \proposed's performance on rare concepts (\eg, less popular identities, new characters). We collected Civitai LoRA models with fewer than $200$ downloads and manually filter out models with common concepts. In total, $81$ Civitai LoRA models have been collected, about $20$ per category. The outlier detector was trained on common concepts and tested on these rare ones. As shown in Table~\ref{tab:civitai_rare}, \proposed achieves comparable or even better performance. This results show that \proposed can be well generalized to rare concepts.

\section{Related Work}

\subsection{Inappropriate concept generation in DMs}
The creative potential of DMs has raised growing concerns about their ability to generate inappropriate concepts. Current approaches to preventing inappropriate concept generation can be categorized into three main strategies.
First, \textit{defensive mechanisms} aim to prevent the generation of certain concepts. For instance, add-on safety filters are employed to detect unsafe input texts or generated images~\cite{Rando2022RedTeamingTS}. However, these filters can often be disabled by users after downloading the models, leaving the core DMs capable of generating inappropriate content. Additionally, safety filters often lack generalization and are susceptible to adversarial attacks, undermining their reliability in diverse scenarios~\cite{Rando2022RedTeamingTS, Yang2023SneakyPromptER}.

Second, \textit{model refining methods} attempt to modify or remove latent representations of inappropriate concepts through techniques like concept removal and machine unlearning~\cite{Schramowski2022SafeLD, Gandikota2023Erasing,Kim2023TowardsSS, Li2024SafeGenMU}. However, these methods are computationally expensive and time-intensive, making them impractical for real-world applications, particularly for DM-sharing platforms or end users who cannot afford the high costs of retraining models.

Third, \textit{prompt probing techniques} have shown promise in auditing concepts by exploring prompts capable of generating specific concepts~\cite{Zhang2023ToGO, Chin2023Prompting4DebuggingRT}. These methods leverage adversarial attacks to optimize prompts,  aligning them with target concepts. For example, prompts can be optimized to match the text embeddings of unsafe target prompts filtered by safety mechanisms~\cite{Yang2023MMADiffusionMA, Chin2023Prompting4DebuggingRT}, or to align with image embeddings of a certain concept~\cite{Wen2023HardPM, Yang2023SneakyPromptJT}. 
However, these techniques rely heavily on accurate prompt probing and robust external detectors to guide the prompt optimization~\cite{Zhang2023ToGO}. Unfortunately, external detectors are not typically trained on generated data, particularly for fine-grained concepts, making them vulnerable to concept drift and biased predictions. %
Our experimental results highlight this limitation: while external detectors perform effectively on real-world images, their reliability degrades significantly when applied to generated content, leading to high false positives.

To the best of our knowledge, there is a critical void of practical and effective tools for large-scale evaluation of DMs.

\subsection{Data memorization analysis}
Data memorization analysis investigates the extent to which diffusion models memorize their training data, primarily through data extraction attacks and membership inference attacks. For example, Data extraction attacks aim to recover exact data samples from the model’s training set~\cite{carlini2023extracting,somepalli2023diffusion,somepalli2023diffusion} while membership inference attacks determine whether a specific data sample was part of the model's training set~\cite{duan2023diffusion,matsumoto2023membership,ma2024could}.
This line of analysis identifies an exact match of the data samples in the DM's training data, which, if successful, provides strong evidence for concept auditing. 
However, applying these data memorization approaches in our work is limited, since deriving precise conclusions for individual data samples often requires extensive querying of the DMs and access to a substantial amount of training data for statistical analysis. These demands make such methods impractical and inefficient for large-scale concept auditing.

\section{Limitations and Discussions}
\subsection{Scope of Concepts}
\label{sec:discussion_concept}
Defining what constitutes a ``concept'' in diffusion models remains an open and underexplored problem. Depending on context, a concept may refer to a broad category, such as ``Disney characters,'' or a highly specific instance, like ``Mickey Mouse.'' This ambiguity is well-documented in related areas such as concept customization~\cite{kumari2023multi,zeng2024improved,smith2023continual} and concept erasure~\cite{Gandikota2023Erasing,lu2024mace,liu2024implicit}, where no consistent standard exists for what qualifies as a distinct or meaningful concept. Such ambiguity poses a challenge for concept auditing, as it complicates the definition of universal auditing criteria.

Rather than imposing a rigid definition, we adopt an \textit{example-based approach}: a concept is considered present if the model can generate recognizable outputs aligned with a small set of reference examples. This flexible design enables our framework to support a wide range of concept granularities and types. In our evaluation, we audit models on concepts such as individual celebrities, complete character sets from specific cartoons (\eg, all Pokémon characters), and object categories from video games. We leave the development of formal taxonomies and clearer conceptual boundaries as an important direction for future work, particularly for improving auditing resolution across different semantic scopes.

While our focus is on semantically meaningful and legally actionable concepts, we note that concept-level auditing for unsafe content (\eg, NSFW or violent imagery) introduces additional ambiguity. Safety-related analysis of diffusion models has gained increasing attention~\cite{schramowski2023safe,zhang2024generate,yang2024sneakyprompt,Li2024SafeGenMU,han2024shielddiff,chen2025safe,li2025detect,chen2025comprehensive}, yet defining and auditing ``unsafe content'' remains difficult due to its context-dependence, vague definitions, and heterogeneous semantics. These factors make unsafe content less applicable to our example-based formulation. While \proposed is not designed to audit unsafe content directly, its methodology could potentially be extended with domain-specific knowledge or safety taxonomies in future work.

\subsection{Auditing in Multi-LoRA Scenarios}
\label{sec:discussion_multilora}
This work mainly focuses on auditing fine-tuned models with a single LoRA module. We investigated the scenarios where multiple concepts are embedded in a single LoRA. Real-world DM deployments also involve {multiple LoRAs}, which are simultaneously loaded into a DM model. This practice is increasingly common on community platforms, where users compose multiple LoRAs, each capturing different aspects of model behavior such as style, texture, or object identity. The multi-LoRA scenario raises questions about how concept representations may be distributed across different LoRAs, and how interactions between LoRAs might influence a model's generative behavior. We view multi-LoRA auditing as a promising direction for extending the \proposed framework.

\subsection{Toward Adversarial Auditing}
\label{sec:discussion_adversary}

This work focuses on \textit{non-adversarial auditing}, where fine-tuned DMs are not deliberately optimized to evade auditing. This setting reflects the majority of real-world cases on public platforms, where users unknowingly or negligently publish potentially problematic DMs, often without consistent metadata or clear prompt disclosure. To assess robustness under adversarial conditions, we include 
five adaptive attacks where the attacker has full knowledge of the auditing framework. \proposed remains effective in both settings, highlighting the robustness of its model-centric design.

Nonetheless, backdoor attacks represent an emerging threat to DMs~\cite{chen2023trojdiff,zhai2023text,chou2023backdoor,chou2024villandiffusion}, where adversaries inject covert patterns into training data or prompts to trigger specific outputs. Existing backdoor detectors primarily focus on defending against image-based triggers~\cite{an2024elijah,mo2024terd,sui2024disdet}. Effective and robust prompt-based trigger detection remains an open challenge. 
Extending \proposed to adversarially evasive fine-tuning introduces a fundamentally different threat model, where the goal is to actively conceal learned concepts. Addressing this challenge would likely require new assumptions, threat models, and detection mechanisms. Moreover, adversarial auditing may conflict with \proposed's scalability goals, as stronger defenses often require expensive or targeted model interrogation. We leave this important but orthogonal direction to future work, and view our current focus as a necessary first step toward practical, large-scale concept auditing.

\section{Conclusion}
We introduce {Prompt-Agnostic Image-Free Auditing (\proposed)}, a model-centric framework for auditing fine-tuned diffusion models without relying on prompt optimization or image-based detection. \proposed combines two key innovations: a \textit{prompt-agnostic design} that analyzes internal model behavior during prompt-insensitive stages of generation, and an \textit{image-free mechanism} based on conditional calibrated error, which compares denoising behavior against the base model to reveal concept learning.

Extensive experiments demonstrate that \proposed consistently outperforms state-of-the-art baselines in both accuracy and efficiency. Our auditing setting assumes internal access to DMs, but does not require access to proprietary fine-tuning data or user prompts, an assumption aligned with how hosted platforms audit uploaded models. On benchmark models, \proposed achieves accurate and efficient concept auditing with significantly reduced computation. On real-world models collected from Civitai, it reaches an average accuracy of {over $92$\%} across categories including celebrities, cartoons, video games, and movies. In addition, \proposed is robust against three adaptive attacks where the attackers know the auditing design, highlighting its robustness and real-world applicability. 

Our results establish \proposed as the first practical and scalable solution for pre-deployment concept auditing in diffusion models. We hope this work contributes to safer, more transparent model sharing and lays the groundwork for future efforts in responsible generative model management.

\section{Acknowledgments}
The work of X. Yuan and L. Zhang was supported in part by the National Science Foundation under Grant 2427316, 2426318. 
The work of L. Guo was sponsored by the Army Research Office and was accomplished under Grant Number W911NF-24-1-0044. The views and conclusions contained in this document are those of the authors and should not be interpreted as representing the official policies, either expressed or implied, of the Army Research Office or the U.S. Government. The U.S. Government is authorized to reproduce and distribute reprints for Government purposes notwithstanding any copyright notation herein.

\bibliographystyle{ACM-Reference-Format}
\balance

\clearpage

\onecolumn
\appendix
\section{Proof of Lemma 1}
\label{sec:proof}
We present the full derivation of the gradient of the cross-attention function with respect to the text embedding $\vp$.

\begin{lemma}
The gradient of the cross-attention function with respect to the text embedding $\vp$ is given by:
\begin{equation}
\frac{\partial Y_i}{\partial \vp} =  (diag(S_i) - S_{i}S_i^T)(\frac{1}{\sqrt{d}}(XW_Q)W_K^T)\vp W_V + S_{i}W_V^T,\end{equation}
where $S$ denotes the cross-attention map (Eq.~\ref{eq:attention_map}), $S_i$ the $i$-th row of $S$, $X$ the image features calculated by the previous layers, and $Y_i$ the $i$-th output of the cross-attention layer (Eq.~\ref{eq:attention_output}). 
\end{lemma}

\begin{proof} 
The cross-attention function can expressed as 
\begin{align*}
Y &=\text{softmax}(\frac{QK^T}{\sqrt{d}})V\\
&= \text{softmax}(\frac{(\vz W_Q) (\vp W_K)^T}{\sqrt{d}}) (\vp W_V)
\end{align*}
Let $A = (X W_Q) (\vp W_K)^T / \sqrt{d}$ and $S = \text{softmax}(A)$, {where}  $S_{ij} = \frac{e^{A_{ij}}}{\sum_k e^{A_{ik}}}$. 
Therefore, we have $Y = S \vp W_V$.
To compute the gradient of $Y$ with respect to $\vp$, we use the chain rule:
\begin{equation}
\label{eq:gradient_split}
\frac{\partial Y}{\partial \vp} = \frac{\partial S}{\partial \vp} (\vp W_V) + S\frac{\partial \vp W_V}{\partial \vp}
\end{equation}
To calculate the first term, we have the Jacobian matrix of Softmax, which is given by:
\begin{equation}
\frac{\partial S_{ij}}{\partial A_{lk}} = 
\begin{cases}
0 & i \neq l \\
S_{ij} (\delta_{jk} - S_{jk}) & i = l
\end{cases}
\end{equation}
where $\delta_{jk}$ is the Kronecker delta, $\delta_{jk} = 1$ if $j=k$ and $0$ otherwise. 
For simplicity, we denote the Jacobian of $S_i$ with respect to $A$ as:
\begin{equation}
J_i :=  \frac{\partial S_i}{\partial A} = diag(S_i) - S_iS_i^T.
\end{equation}

Thus, the gradient in the first term of Eq.~\ref{eq:gradient_split} becomes: 
\begin{equation}
\label{eq:softmax_gradient}
    \frac{\partial S_i}{\partial \vp} = \frac{\partial S_i}{\partial A} \frac{\partial A}{\partial \vp}  = J_i(\frac{1}{\sqrt{d}}(XW_Q)W_K^T).
\end{equation}

By incorporating Eq.~\ref{eq:softmax_gradient} into Eq.~\ref{eq:gradient_split}, we have:
\begin{equation*}
\frac{\partial Y_i}{\partial \vp} =  J_i(\frac{1}{\sqrt{d}}(XW_Q)W_K^T) \vp W_V + S_iW_V^T
\end{equation*}
\end{proof}

\section{\proposed Pseudocode}
We present a complete pseudocode of \proposed in Algorithm~\ref{alg:paia}.

\begin{center}
\begin{minipage}{0.8\linewidth} 
\begin{algorithm}[H]
\caption{PAIA: Prompt-Agnostic Image-Free Auditing (Training and Inference)}
\label{alg:paia}
\KwIn{
Target model $f_{W'}$,  Base model $f_W$, Image set with the target concept $\mathcal{D}_{\text{target}}$, Image set with irrelevant concepts $\mathcal{D}_{\text{irr}}$, Cutoff time step $\tau$, Time steps $\mathcal{T}$
}
\KwOut{Prediction $y \in \{0,1\}$ indicating whether $f_{W'}$ can generate the target concept}

\vspace{0.5em}
Freeze the cross-attention in $f_{W'}$ to obtain parameters $W''$

\vspace{0.5em}
\SetKwFunction{CCE}{ComputeFeatures}
\SetKwProg{Fn}{Function}{:}{}
\Fn{\CCE{$f_{W'}$, $\mathcal{D}$}}{
    Initialize feature set $\mathcal{F} \leftarrow \emptyset$ \\
    \ForEach{$\bm{x} \in \mathcal{D}$}{
        \ForEach{$t \in \mathcal{T}$}{
            Generate noised latent image representation $\vz_t$ and pseudo-prompt $\vp$ from $\bm{x}$ \\
            \uIf{$t < \tau$}{
                $\mathcal{L}_{cce}^{t} \leftarrow 
                    \mathbb{E} \left[\| {\vepsilon}_{W'}(\vz_t, \vp) - \vepsilon_0 \|^2 \right] 
                    - \| {\vepsilon}_{W}(\vz_t, \vp) - \vepsilon_0 \|^2$
            }
            \Else{
                $\mathcal{L}_{cce}^{t} \leftarrow 
                    \mathbb{E} \left[\| {\vepsilon}_{W''}(\vz_t, \vp) - \vepsilon_0 \|^2 \right] 
                    - \| {\vepsilon}_{W}(\vz_t, \vp) - \vepsilon_0 \|^2$
            }
        }
        $\phi(\bm{x}) \leftarrow \{ \mathcal{L}_{cce}^{t} \}_{t \in \mathcal{T}}$ \\
        $\mathcal{F} \leftarrow \mathcal{F} \cup \{\phi(\bm{x})\}$
    }
    \Return{$\mathcal{F}$}
}

\vspace{0.5em}
\textbf{Training Phase:} \\
$\mathcal{F}_{\text{irr}} \leftarrow$ \CCE{$f_{W'}$, $\mathcal{D}_{\text{irr}}$} \\
Train an anomaly detector $\mathcal{A}$ on $\mathcal{F}_{\text{irr}}$

\vspace{0.5em}
\textbf{Inference Phase:} \\
$\mathcal{F}_{\text{target}} \leftarrow$ \CCE{$f_{W'}$, $\mathcal{D}_{\text{target}}$} \\
Initialize anomaly indicator list $Y \leftarrow \emptyset$ \\
\ForEach{$\phi(\bm{x}_i) \in \mathcal{F}_{\text{target}}$}{
    Compute anomaly score $s_i \leftarrow \mathcal{A}(\phi(\bm{x}_i))$ \\
    $y_i \leftarrow \mathbb{I}[s_i > \delta]$ // 1 if the anomaly score is great than threshold $\delta$\\
    $Y \leftarrow Y \cup \{y_i\}$
}
\Return{$y = \text{MajorityVote}(Y)$}
\end{algorithm}
\end{minipage}
\end{center}

\section{Terms used in random prompt generation.}
\label{sec:prompt_example}
The following set of descriptive keywords was used to augment prompts in the random-prompt evaluation setting. These terms were randomly inserted to simulate realistic prompt noise and diversity.

\begin{table}[!ht]
\caption{Words used in pseudo prompt generation strategy.}
\label{tab:random_term}
\begin{tabular}{|cccc|}
\hline
natural lighting & portrait       & photorealistic    & best quality \\
realistic        & ultra detailed & standing          & highres      \\
detailed face    & solo           & masterpiece       & outdoors     \\
film grain       & illustration   & soft light        & raw photo    \\
street           & from side      & looking at viewer & sitting      \\ \hline
\end{tabular}

\end{table}

\section{Benchmarking Dataset Samples: Celebrities and Cartoons}
\label{sec:benchmark}
We provide example images from the curated datasets used in our benchmark evaluation in Figure~\ref{fig:celebrity_examples} and~\ref{fig:cartoon_examples}. Each concept is represented by a small set of reference images used during fine-tuning and for auditing evaluation.

\begin{figure}[!h]
\centering
\includegraphics[width=0.9\linewidth]{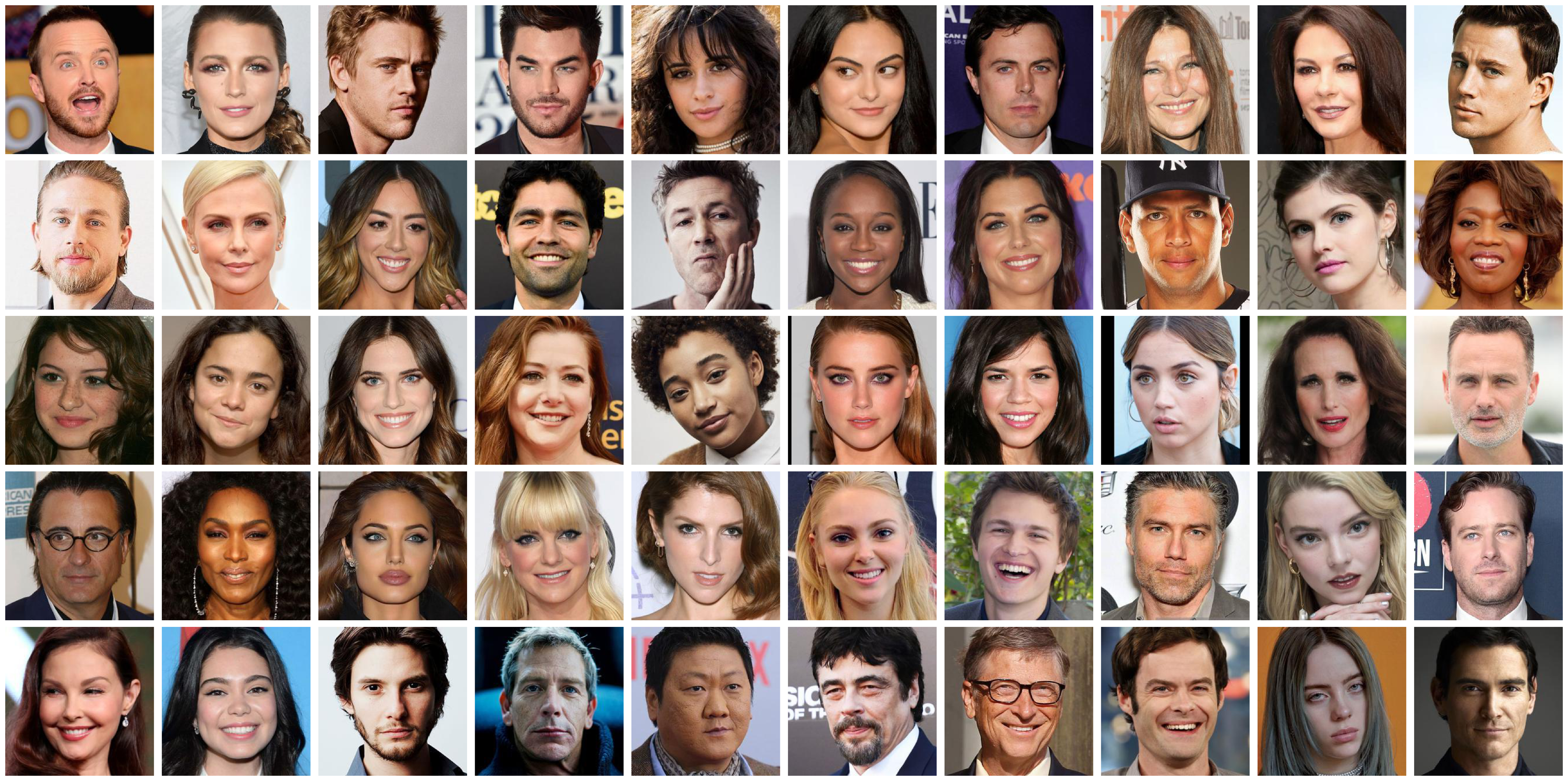}
\caption{Example concept images from the Celebrity dataset used in benchmarking.}
\label{fig:celebrity_examples}
\end{figure}

\begin{figure}[!h]
\centering
\includegraphics[width=0.9\linewidth]{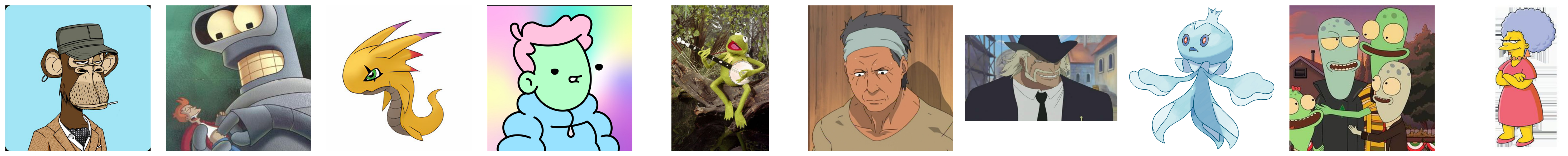}
\caption{Example concept images from the Cartoon character dataset used in benchmarking.}
\label{fig:cartoon_examples}
\end{figure}

\section{Civitai model analysis.}
\label{sec:app_civitai}
To support real-world evaluation, we analyze the diffusion models hosted on Civitai, one of the largest community platforms for sharing LoRA-based fine-tuned models.

Figure~\ref{fig:civitai_tag} shows the distribution of commonly used tags on Civitai. Among these, the most frequent concept categories include ``celebrity,'' ``game,'' and ``cartoon.'' Based on this trend, our evaluation focuses on LoRA models falling within these categories. Additionally, we include a ``movie'' category to capture models, which overlaps with tags such as ``character'' and ``actress''.

Figure~\ref{fig:civitai_base_model} shows that Stable Diffusion 1.5 (SD1.5) is the most commonly used base model for LoRA fine-tuning. Therefore, our real-world evaluation focuses primarily on LoRA models adapted from SD1.5.

\begin{figure}[!h]
\centering
\includegraphics[width=0.5\linewidth]{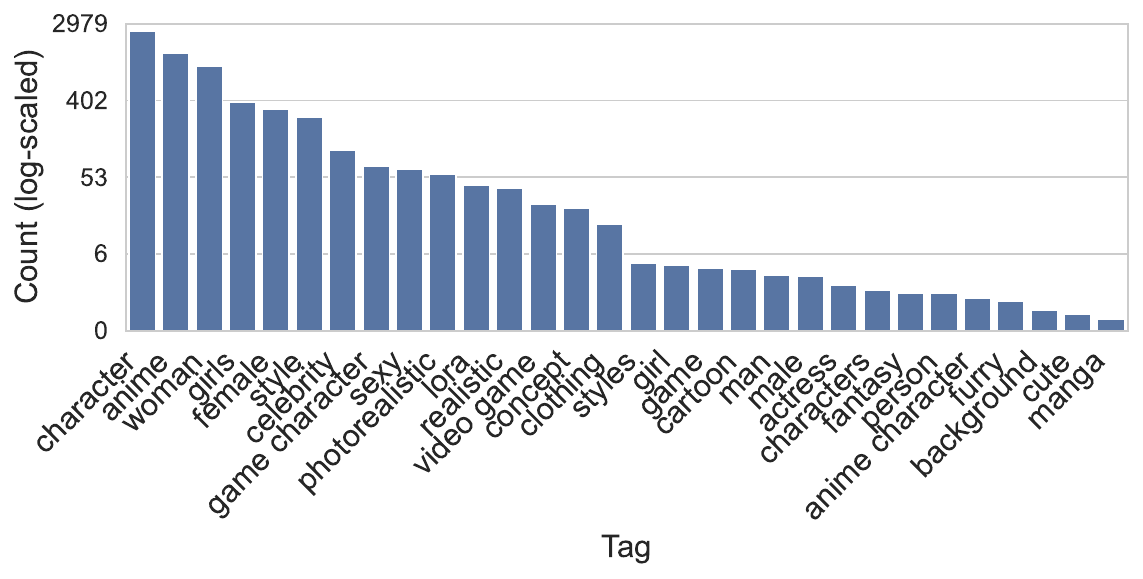}
\caption{Common tags in models uploaded to Civitai.}
\label{fig:civitai_tag}
\end{figure}

\begin{figure}[!h]
\centering
\includegraphics[width=0.3\linewidth]{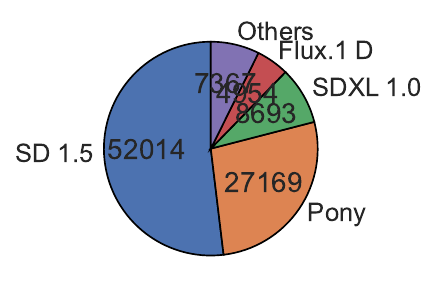}
\caption{Distribution of base models used for LoRA fine-tuning on Civitai.}
\label{fig:civitai_base_model}
\end{figure}

\section{Additional Evaluation Results}
\label{sec:eval_app}
\subsection{Impact of Cutoff Time Step.}
\begin{figure}[!h]
\centering
\includegraphics[width=0.3\linewidth]{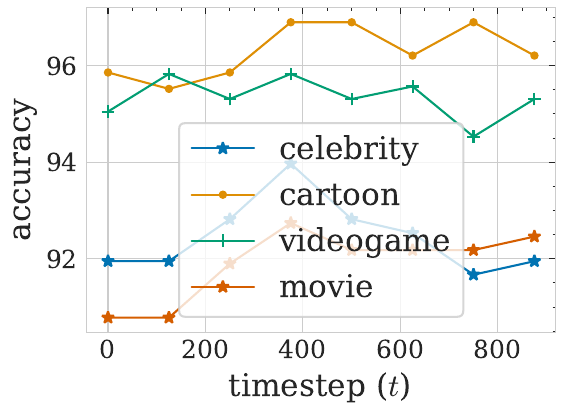}
\caption{Impact of cutoff time step.}
\label{fig:cutoff}
\end{figure}

We evaluate the impact of varying the cutoff timestep $\gamma$ in CCE (Eq.~\ref{eq:final_loss}) using Civitai models, with $\gamma \in \{\frac{T}{8}, \frac{T}{4}, \frac{3T}{8}, \frac{T}{2}, \frac{5T}{8}, \frac{3T}{4}, \frac{7T}{8}, T\}$. As shown in Figure~\ref{fig:cutoff}, the best performance is achieved around $\gamma = \frac{T}{2}$, balancing prompt suppression and behavioral signal strength. Early cutoffs underutilize prompt information, while late cutoffs weaken calibration. Performance across mid-range values (\eg, between $\frac{3T}{8}$ and $\frac{5T}{8}$) is relatively stable , indicating that \proposed is not sensitive to the exact choice of $\gamma$.

\subsection{Effectiveness of Conditional Calibration on Civitai Models}

\begin{table}[!h]
\caption{Impact of Conditional Calibration on Civitai Models.}
\label{tab:selective_civitai}
\small
\begin{tabular}{@{}lrr@{}}
\toprule
Category & With Conditional Calibration & Without Conditional Calibration \\ \midrule
Celebrity & 93.97\% & 91.38\% \\
Cartoon   & 96.90\% & 95.86\% \\
Videogame & 95.83\% & 95.05\% \\
Movie     & 92.74\% & 91.06\% \\ \bottomrule
\end{tabular}%
\end{table}

We investigate the effectiveness of conditional calibration on Civitai Models. 
As shown in Table~\ref{tab:selective_civitai}, conditional calibration improves auditing accuracy across all categories. Although the gains are relatively small (1–2\%), they are consistent and meaningful, especially given the strong baseline performance without calibration.

\section{Trigger Quality Analysis on Civitai Models}
\label{sec:app_trigger_analysis}

To better understand the challenges in prompt-based auditing, we analyze the trigger words associated with LoRA models uploaded to Civitai. While many models provide trigger prompts, the quality and alignment of these prompts vary significantly.

\begin{figure}[!h]
\centering
\begin{subfigure}[b]{0.45\linewidth}
\centering
\includegraphics[width=0.7\linewidth]{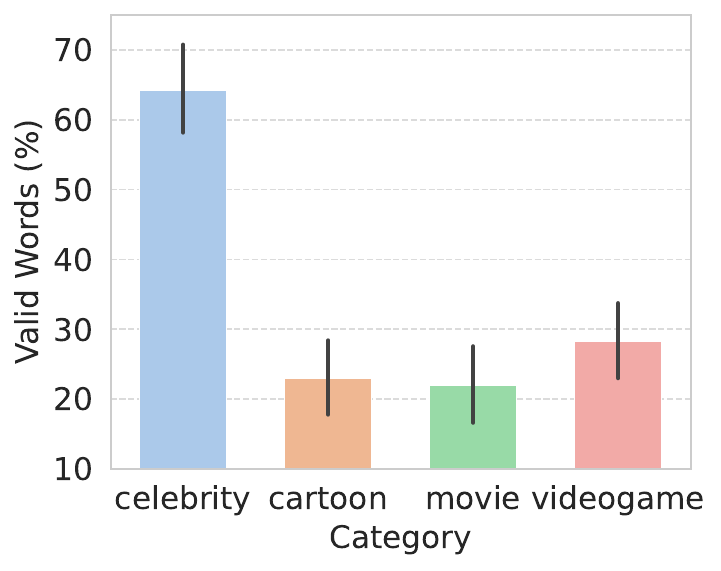}
\caption{Percentage of valid English words used in trigger words.}
\label{fig:civitai_word}
\end{subfigure}
\hfill
\begin{subfigure}[b]{0.45\linewidth}
\centering
\includegraphics[width=0.7\linewidth]{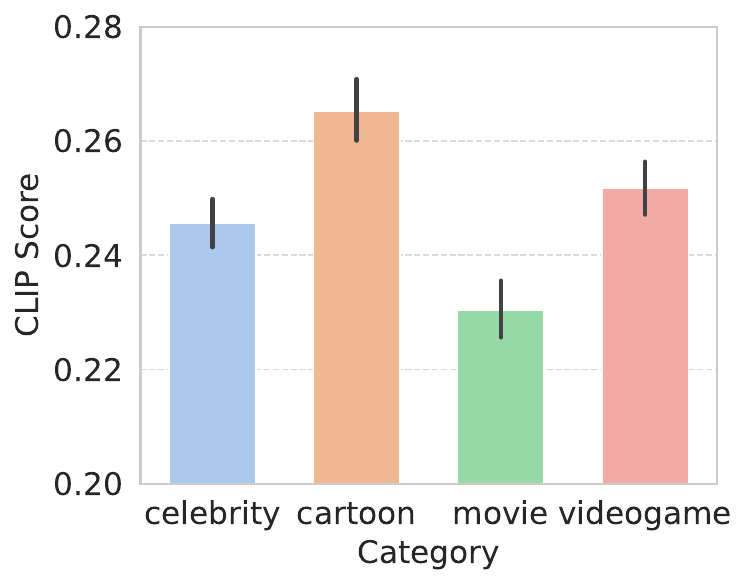}
\caption{CLIP similarity between generated images and their associated trigger words.}
\label{fig:civitai_clip_score}
\end{subfigure}
\caption{Analysis of trigger words that are associated with uploaded DMs over four categories on Civitai.}
\end{figure}

Figure~\ref{fig:civitai_word} shows that a large portion of trigger words used by creators are not valid English words. These nonstandard or synthetically generated tokens are common in community models, making them difficult to probe or interpret.

Figure~\ref{fig:civitai_clip_score} further highlights the issue by quantifying the semantic alignment between trigger words and their corresponding images. We use CLIP similarity scores to measure how well a trigger word matches the content of its sample image. The observed scores reveal that many triggers exhibit weak semantic correspondence with the generated outputs, indicating that even when a trigger is known, it may not meaningfully reflect the learned concept.

Together, these findings underscore a key obstacle in prompt-based auditing—\textbf{trigger uncertainty}, which arises from the large, discrete nature of the prompt space and the inherent ambiguity of natural language. Optimization-based probing methods (\eg, adversarial or reinforcement learning) often fail to reliably discover effective prompts, especially in the presence of low-quality or misleading triggers. These issues further motivate our model-centric, prompt-agnostic auditing framework.

\end{document}